\documentclass{article}

\usepackage[final, nonatbib]{nips_2017}
\usepackage[numbers]{natbib}

\usepackage[utf8]{inputenc} 
\usepackage[T1]{fontenc}    
\usepackage{hyperref}       
\usepackage{url,xspace}            
\usepackage{booktabs}       
\usepackage{amsfonts}       
\usepackage{nicefrac}       
\usepackage{microtype}      
\usepackage{graphicx}
\usepackage{amsmath}
\usepackage{amsthm}
\usepackage{amssymb}

\newtheorem{remark}{Remark}
\usepackage{algorithm}
\usepackage{algorithmic}
\usepackage{color}

\title{Learning Mixture of Gaussians with Streaming Data}

\newcommand\est[2][]{{\boldsymbol{\mu}}_{#1}^{#2}}
\newcommand\kest[2]{{\boldsymbol{\mu}}_{#1}^{#2}}

\newcommand\true[1][]{\ensuremath{\boldsymbol{\mu}_{#1}^\star}}
\newcommand \ktrue[1]{\ensuremath{\boldsymbol{\mu}^\star_{#1}}}

\newcommand\sI{\ensuremath{\mathcal{I}}}

\newcommand\sN{\ensuremath{\mathcal{N}}}

\newcommand\bs[1]{\mathbf{#1}\xspace}

\newcommand{\poly}{\text{poly}}

\newcommand\R{\ensuremath{\mathbb{R}}} 
\newcommand\eqdef{\ensuremath{\stackrel{\rm def}{=}}} 

\ifthenelse{\isundefined{\definition}}{\newtheorem{definition}{Definition}}{}
\ifthenelse{\isundefined{\assumption}}{\newtheorem{assumption}{Assumption}}{}
\ifthenelse{\isundefined{\hypothesis}}{\newtheorem{hypothesis}{Hypothesis}}{}
\ifthenelse{\isundefined{\proposition}}{\newtheorem{proposition}{Proposition}}{}
\ifthenelse{\isundefined{\theorem}}{\newtheorem{theorem}{Theorem}}{}
\ifthenelse{\isundefined{\fact}}{\newtheorem{fact}{Fact}}{}
\ifthenelse{\isundefined{\lemma}}{\newtheorem{lemma}{Lemma}}{}
\ifthenelse{\isundefined{\corollary}}{\newtheorem{corollary}{Corollary}}{}
\ifthenelse{\isundefined{\alg}}{\newtheorem{alg}{Algorithm}}{}
\ifthenelse{\isundefined{\example}}{\newtheorem{example}{Example}}{}
\newcommand{\E}{\ensuremath{\mathbb{E}}} 

\newcommand\ie{\textit{i.e.}\ }

\newcommand{\prtrue}[1]{\widehat{\boldsymbol{\mu}}^*_{#1}}

\author{
  Aditi Raghunathan \\
  Stanford University \\
  \texttt{aditir@stanford.edu} \\
  \And
  Ravishankar Krishnaswamy \\
  Microsoft Research, India \\
  \texttt{rakri@microsoft.com} \\
  \And
  Prateek Jain \\
  Microsoft Research, India\\
  \texttt{prajain@microsoft.com}}
\begin{document}

\maketitle

\begin{abstract}
In this paper, we study the problem of learning a mixture of Gaussians with streaming data: given a stream of $N$ points in $d$ dimensions generated by an unknown mixture of $k$ spherical Gaussians, the goal is to estimate the model parameters using a single pass over the data stream. We analyze a streaming version of the popular Lloyd's heuristic and show that the algorithm estimates all the unknown centers of the component Gaussians  accurately if they are sufficiently separated. Assuming each pair of centers are $C\sigma$ distant with $C=\Omega((k\log k)^{1/4}\sigma)$ and where $\sigma^2$ is the maximum variance of any Gaussian component, we show that asymptotically the algorithm estimates the centers optimally (up to certain constants); our center separation requirement matches the best known result for spherical Gaussians \citep{vempalawang}. For finite samples, we show that a bias term based on the initial estimate decreases at $O(1/{\rm poly}(N))$ rate while variance decreases at nearly optimal rate of $\sigma^2 d/N$.

Our analysis requires seeding the algorithm with a good initial estimate of the true cluster centers for which we provide an online PCA based clustering algorithm. Indeed, the asymptotic per-step time complexity of our algorithm is the optimal $d\cdot k$ while space complexity of our algorithm is $O(dk\log k)$.

In addition to the bias and variance terms which tend to $0$, the hard-thresholding based updates of streaming Lloyd's algorithm is agnostic to the data distribution and hence incurs an \emph{approximation error} that cannot be avoided. However, by using a streaming version of the classical \emph{(soft-thresholding-based)} EM method that exploits the Gaussian distribution explicitly, we show that for a mixture of two Gaussians the true means can be estimated consistently, with estimation error decreasing at nearly optimal rate, and tending to $0$ for $N\rightarrow \infty$.
\end{abstract}

\section{Introduction}
Clustering data into homogeneous clusters is a critical first step in any data analysis/exploration task and is used extensively to pre-process data, form features, remove outliers and visualize data. Due to the explosion in amount of data collected and processed, designing clustering algorithms that can handle \emph{large datasets} that do not fit in RAM is paramount to any big-data system. A common approach in such scenarios is to treat the entire dataset as a \emph{stream} of data, and then design algorithms which update the model after every few points from the data stream. In addition, there are several practical applications where the data itself is not available beforehand and is streaming in, for example in any typical online system like web-search.

For such a model, the algorithm of choice in practice is the so-called \emph{streaming $k$-means heuristic}. It is essentially a streaming version of the celebrated $k$-means algorithm or Lloyd's heuristic \cite{DudaHart2nd}. The basic $k$-means algorithm is designed for offline/batch data where each data point is assigned to the nearest centroid and the centroids are then updated based on the assigned points;  this process is iterated till the solution is locally optimal. The streaming version of the $k$-means algorithm assigns the new point from the stream to the closest centroid and only updates this centroid {\em immediately}. That is, unlike offline $k$-means which first assigns all the points to the respective centroids and then updates the centroids, the streaming algorithm updates the centroids after each point, making it much more space efficient. While  streaming $k$-means and its several variants are used heavily in practice, their properties such as solution quality, time complexity of convergence have not been studied widely. In this paper, we attempt to provide one of the first such theoretical studies of the streaming $k$-means heuristic. One of the big challenges in such a study is that even the (offline) $k$-means algorithm attempts to solve a non-convex NP-hard problem. Streaming data poses additional challenges due to large noise in each point that can deviate the solution significantly.

In the offline setting, clustering algorithms are typically studied under certain simplifying assumptions that bypasses the worst-case NP-hardness of these problems. One of the most widely studied setting is when the data is sampled from a \emph{mixture of well-separated Gaussians} \cite{dasgupta2007spectral,vempalawang,anandkumar}, which is also the generative assumption that we impose on the data. However, the online/streaming version of the $k$-means algorithm has not been studied in such settings. In this work, we design and study a variant of the popular online $k$-means algorithm where the data is streaming-in, we cannot store more than logarithmically many data points, and each data point is sampled from a mixture of well-separated spherical Gaussians. The goal of the algorithm is then to learn the means of each of the Gaussians; note that estimating other parameters like variance, and weight of each Gaussian in the mixture becomes simple once the true means are estimated accurately.

{\bf Our Results.} Our main contribution is the {\em first} bias-variance  bound for the problem of clustering with streaming data.
Assuming that the centers are separated by $C \sigma$ where $C = \Omega(\sqrt{\log k})$ and if we seed the algorithm with initial cluster centers that are $\leq C\sigma /20$ distance away from the true centers, then we show that the error in estimating the true centers can be decomposed into three terms and bound each one of them: {\bf (a)} the bias term, i.e., the term dependent on distance of true means to initial centers   decreases at a $1/{\rm poly}(N)$ rate, where $N$ is the number of data points observed so far, {\bf (b)} the variance term is bounded by $\sigma^2 \big (\frac{d \log N}{N} \big)$ where $\sigma$ is the standard deviation of each of the Gaussian, $d$ is the dimensionality of the data, and {\bf (c)} \emph{offline approximation error}: indeed, note that even the offline Lloyd's heuristic will have an approximation error due to its hard-thresholding nature. For example, even when $k=2$, and the centers are separated by $C \sigma$, around $\exp(-\frac{C^2}{8})$ fraction of points from the first Gaussian will be closer to the second center, and so the $k$-means heuristic will converge to centers that are at a squared distance of roughly $O(C^2)  \exp(-\frac{C^2}{8}) \sigma^2$ from the true means. We show that we essentially almost inherit this optimal offline error  in the streaming setting. 
Note that the above result holds at a center separation of $\Omega(\sqrt{\log k}\sigma)$ distance, which is substantially weaker than the currently best-known result of $\Omega (\sigma k^{1/4})$ for even the offline problem \cite{vempalawang}. However, as mentioned before, this only holds provided we have a good initialization. To this end, we show that when $C = \Omega(\sigma (k\log k)^{1/4})$, we can combine an \emph{online PCA} algorithm~\cite{HardtP14, jainjkns} with the batch $k$-means algorithm on a small seed sample of around $O(k \log k)$ points, to get such an initialization. Note that this separation requirement nearly matches the best-known result offline results~\cite{vempalawang}. 
Finally, we also study a soft-version of streaming $k$-means algorithm, which can also be viewed as the streaming version of the popular Expectation Maximization (EM) algorithm. We show that for mixture of two well-separated Gaussians, a variant of streaming EM algorithm recovers the above mentioned bias-variance bound but {\em without} the approximation error. That is, after observing infinite many samples, streaming EM converges to the true means and matches the corresponding offline results in \cite{balakrishnan2014statistical,daskalakis2016ten}; to the best of our knowledge this is also first such consistency result for the streaming mixture problem. However, the EM updates require that the data is sampled from mixture of Gaussians, while the updates of streaming Lloyd's algorithm are agnostic of the data distribution and hence same updates can be used to solve arbitrary mixture of sub-Gaussians as well.

{\bf Technical Challenges.} One key technical challenge in analyzing steaming $k$-means algorithm in comparison to the standard streaming regression style problems is that the offline problem itself is non-convex and moreover can only be solved approximately. Hence, a careful analysis is required to separate out the error we get in each iteration in terms of the bias, variance, and inherent approximation error terms. Moreover, due to the non-convexity, we are able to guarantee decrease in error only if each of our iterates lies in a small ball around the true mean. While this is initially true due to the initialization algorithm, our intermediate centers might escape these balls during our update. However, we show using a delicate martingale based argument that with high probability, our estimates stay within \emph{slightly larger balls} around the true means, which turns out to be sufficient for us.

{\bf Related Work.}
A closely related work to ours is an {\em independent} work by \cite{monteleoni} which studies a stochastic version of $k$-means for data points that satisfy a spectral variance condition which can be seen as a deterministic version of the mixture of distributions assumption. However, their method requires multiple passes over the data, thus doesn't fit directly in the streaming $k$-means setting. In particular, the above mentioned paper analyzes the stochastic $k$-means method only for highly accurate initial set of iterates which requires a large burn-in period of $t=O(N^2)$ and hence needs $O(N)$ passes over the data, where $N$ is the number of data points. Tensor methods \cite{anandkumar,hsukakade} can also be extended to cluster streaming data points sampled from a mixture distribution but these methods suffer from large sample/time complexity and might not provide reasonable results when the data distribution deviates from the assumed generative model.

In addition to the gaussian mixture model, clustering problems are also studied under other models such as data with small spectral variance \cite{kumar2010clustering}, stability of data \cite{balcan}, etc. It would be interesting to study the streaming versions in such models as well.

\textbf{Paper Outline.} We describe our models and problem setup in Section~\ref{sec:setup}. We then present our streaming $k$-means algorithm, and its proof overview in Sections~\ref{sec:algo} and~\ref{sec:analysis}. We then discuss the initialization procedure in Section~\ref{sec:init}. Finally we conclude with details of our streaming-EM algorithm in Section~\ref{sec:soft}. The full version of this paper appears in the supplementary material.

\section{Setup and Notation}
\label{sec:setup}
We assume that the data is drawn from a mixture of $k$
spherical Gaussians distributions, i.e.,
\begin{equation}\label{eq:xt}
\bs{x^t}\stackrel{i.i.d}{\sim} \sum_i w_i \sN(\ktrue{i}, \sigma^2I), \ktrue{i} \in \R^d~ \forall i = 1, 2, \hdots k
\end{equation}
where $\ktrue{i}\in \R^d$ is the mean of the $i$-th mixture component, mixture weights $w_i\geq 0$, and $\sum_i w_i=1$.
All the problem parameters (i.e., the true means, the variance $\sigma^2$ and the mixture weights) are unknown to the algorithm.
Using the standard streaming setup, where the $t^{th}$ sample $\bs{x^t} \in \R^d$ is drawn from
the data distribution, our goal is to produce an estimate $\hat{\bs{\mu}}_i$ of $\ktrue{i}$
for $i = 1, 2, \hdots k$ in a single pass over the data using bounded space.

{\bf Center Separation.} A suitable notion of signal to noise ratio for our problem turns out to be the ratio of minimum separation between the true centers and the maximum variance along any direction. We denote this ratio by $C =\min_{i, j} \frac{ \|\ktrue{i} - \ktrue{j} \|}{\sigma}$. For convenience, we also denote $\frac{\| \ktrue{i} - \ktrue{j} \|}{\sigma}$ by $C_{ij}$. Here and in the rest of the paper, $\| \bs{y} \|$ is the Euclidean norm of a vector $\bs{y}$. We use $\eta$ to denote the learning rate of the streaming updates and $\kest{i}{t}$ to denote the estimate of $\ktrue{i}$ at time $t$. 
For a cleaner presentation, we assume that all the mixture weights are $1/k$, but our results hold with general weights as long as an appropriate  center separation condition is satisfied. We omit these details in this presentation.

\section{Algorithm and Main Result}\label{sec:algo}
\label{sec:algorithm}
In this section, we describe our proposed streaming clustering algorithm and present our analysis of the algorithm. At a high level, we follow the approach of various recent results for (offline) mixture recovery algorithms \cite{vempalawang, kumar2010clustering}. That is, we initialize the algorithm with an SVD style operation which de-noises the data significantly and then apply Lloyd's heuristic (in an online manner). Note that the Lloyd's algorithm is agnostic to the underlying distribution and does not include distribution specific terms like variance etc.  Algorithm~\ref{alg:streaming} presents a pseudo-code of our algorithm. Note that for initialization we use the {\sf InitAlg} subroutine. 

Intuitively, the initialization algorithm first computes an online batch PCA in the for-loop. After this step, we perform an offline distance-based clustering on the projected subspace (akin to Vempala-Wang for the offline algorithm). This only uses few (roughly $k \log k$) samples since we only need estimates for centers within a suitable proximity from the true centers. The centers output are fed as the initial centers for the streaming update algorithm, which then, for each new sample, updates the current center which is closest to the sample, and iterates.

\begin{minipage}[t]{.5\textwidth}
	\begin{algorithm}[H]
		\caption{{\sf StreamKmeans}$(N, N_0)$}
		\label{alg:streaming}
		\begin{algorithmic}[1]
			\STATE {\bf Set} $\eta \leftarrow \frac{3k \log 3N}{N}$.
			\STATE {\bf Set} $\{\est[1]{0}, \dots, \est[k]{0}\}\leftarrow {\sf InitAlgo}(N_0)$.
			\FOR{$t = 1$ to $N$}
			\STATE Receive $\bs{x}^{t+N_0}$ given by the input stream
			\STATE $\bs{x}=\bs{x}^{t+N_0}$
			\STATE {\bf Let} $i_t = \arg \min_{i} \| \bs{x} - \est[i]{t-1} \|$.
			\STATE {\bf Set} $\est[i_t]{t} = (1-\eta) \est[i_t]{t-1} + \eta \bs{x}$
			\STATE {\bf Set} $\est[i]{t} = \est[i]{t-1}$ for $i \neq i_t$
			\ENDFOR
			\STATE Output: {$\est[1]{N}, \dots, \est[k]{N}$}
		\end{algorithmic}\vspace*{1.68cm}
	\end{algorithm}
\end{minipage}
\begin{minipage}[t]{.5\textwidth}
	\begin{algorithm}[H]
		\caption{{\sf InitAlg}$(N_0)$}
		\label{alg:init}
		\begin{algorithmic}
			\STATE $U\leftarrow$ random orthonormal matrix $\in \R^{d\times k}$
			\STATE $B=\Theta(d\log d)$, $S=0$
			\FOR{$t = 1$ to $N_0-k\log k$}
			\IF{$\text{mod}(t, B)=0$}
			\STATE $U \leftarrow QR(S\cdot U),\quad$ $S \leftarrow 0$
			\ENDIF
			\STATE Receive $\bs{x}^t$ as generated by the input stream
			\STATE $S=S+\bs{x}^t(\bs{x}^t)^T$
			\ENDFOR
			\STATE $X_{0}=[\bs{x}^{N_0-k\log k+1},  \dots, \bs{x}^{N_0}]$
			\STATE Form nearest neighbor graph using $U^TX_0$ and find connected components
			\STATE $[\bs{\nu}^0_1, \hdots, \bs{\nu}^0_k]\leftarrow $ mean of points in each component
			\STATE \textbf{Return:} $[\kest{1}{0}, \hdots, \kest{k}{0}]=[U\bs{\nu}^0_1, \hdots, U\bs{\nu}^0_k]$
		\end{algorithmic}
	\end{algorithm}
\end{minipage}

We now present our main result for the streaming clustering problem.
\begin{theorem}\label{thm:main}
	Let ${\bs{x^t}}$, $1\leq t\leq N+N_0$ be generated using a mixture of Gaussians \eqref{eq:xt} with $w_i=1/k$, $\forall i$. Let $N_0, N \geq O(1)k^3 d^3\log d$ and $C\geq \Omega((k\log k)^{1/4})$. Then, the mean estimates $(\est[1]{N}, \dots, \est[k]{N})$ output by Algorithm~\ref{alg:streaming} satisfies the following error bound:
	\[\E \left[ \sum_{i} \|\est[i]{N} - \true[i]\|^2 \right] \leq \underbrace{\frac{  \max_{i} \|\true[i]\|^2 }{\rm N^{\Omega(1)}}}_{{\rm bias}} +  O(k^3) \left( \underbrace{\sigma^2\frac{d \log N}{N} }_{\rm variance}+\underbrace{ \exp(-C^2/8) (C^2 + k) \sigma^2}_{\approx {\rm offline \,}k{\rm-means \, error}}  \right).\]

\end{theorem}

Our error bound consists of three key terms: {\em bias, variance, and offline $k$-means error}, with bias and variance being standard statistical error terms: (i) bias is dependent on the initial estimation error and goes down at $N^\zeta$ rate where $\zeta>0$ is a large constant; (ii) variance error is the error due to noise in each observation ${\bs{x^t}}$ and goes down at nearly optimal rate of $\approx \sigma^2 \frac{d}{N}$ albeit with an extra $\log N$ term as well as worse dependence on $k$; and (iii) an offline $k$-means error, which is the error that even the offline Lloyds' algorithm would incur for a given center separation $C$. Note that while sampling from the mixture distribution, $\approx \exp(-C^2/8)$ fraction of data-points can be closer to the true means of other clusters rather than their own mean. Hence, in general it is not possible to assign back those points to that cluster and hence will lead to estimation error for the hard assignment based Lloyd's heuristic. See Figure~\ref{fig:kerror} for an illustration. This error can however be avoided by performing soft updates, which is discussed in Section~\ref{sec:soft}.

{\bf Remarks.} Even if the weights $w_i\neq 1/k$, our algorithm remains the same and our analysis goes through with simple modifications. Secondly, our proofs also follow when the gaussians have different $\sigma_i$, as long as the necessary conditions are satisfied with $\sigma=\max_i \sigma_i$.

{\bf Time, space, and sample complexity}: Our algorithm has nearly optimal time complexity of $O(d\cdot k)$ per iteration; the initialization algorithm requires about $O(d^4 k^3)$ time. Space complexity of our algorithm is $O(d k \cdot \log k)$ which is also nearly optimal. Finally, the sample complexity is $O(d^3 k^3)$, which is a loose upper bound and can be significantly improved by a more careful analysis.

{\bf Analysis Overview.}
The proof of Theorem ~\ref{thm:main} essentially follows from the two theorems stated below: a) update analysis \emph{given a good initialization}; b) {\sf InitAlg} analysis for showing such an initialization.
\begin{theorem}[Streaming Update]
	\label{thm:kthm}
	Let $\bs{x^t}$, $1\leq t\leq N+N_0$ be generated using a mixture of Gaussians \eqref{eq:xt} with $w_i=1/k$, $\forall i$, and $N = \Omega(k^3 d^3 \log kd)$. Also, let the center-separation $C \geq \Omega(\sqrt{\log k})$, and also suppose our initial centers $\est[i]{0}$ are such that for all $1 \leq i \leq k$, $\| \est[i]{0} - \true[i] \| \leq \frac{C \sigma}{20}$.
	
	Then, the streaming update of {\sf StreamKmeans}$(N, N_0)$ , i.e, Steps 3-8 of Algorithm~\ref{alg:streaming} satisfies: \[\E \left[ \sum_{i} \|\est[i]{N} - \true[i]\|^2 \right] \leq {\frac{\max_{i} \|\true[i]\|^2 }{\rm N^{\Omega(1)}}} +  O(k^3) \left( { \exp(-C^2/8) (C^2 + k) \sigma^2} + {\frac{ \log N}{N} d \sigma^2} \right).\]
\end{theorem}
Note that our streaming update analysis requires only $C=\Omega(\sqrt{\log k})$ separation but needs appropriate initialization that is guaranteed by the below result.
\begin{theorem}[Initialization]
	\label{thm:initialization}
	Let $\bs{x^t}$, $1\leq t\leq N_0$ be generated using a mixture of Gaussians \eqref{eq:xt} with $w_i=1/k$, $\forall i$. Let $\kest{1}{0}, \kest{2}{0}, \hdots \kest{k}{0}$ be the output of Algorithm~\ref{alg:init}.
	If $C = \Omega \Big ( (k\log k)^{1/4}\Big )$
	and $N_0= \Omega \Big ( d^3 k^3 \log dk \Big)$, then w.p. $\geq 1-1/poly(k)$, we have $\max_i\| \kest{i}{0} - \ktrue{i} \| \leq \frac{C}{20}\sigma$
\end{theorem}

\section{Streaming Update Analysis} \label{sec:analysis}
At a high level our analysis shows that at each step of the streaming updates, the error decreases on average. However, due to the non-convexity of the objective function we can show such a decrease only if the current estimates of our centers lie in a small ball around the true centers of the gaussians. Indeed, while the initialization provides us with such centers, due to the added noise in each step, our iterates may occasionally fall outside these balls. To overcome this, we use a careful Martingale based argument to show that, with high probability, the candidate centers maintained by our algorithm lie in \emph{slightly larger balls} around the true centers \emph{in every iteration} of the algorithm. We therefore divide our proof in two parts: a) first we show in Section~\ref{sec:single} that the error decreases in expectation, assuming that the current estimates lie in a reasonable neighborhood around the true centers; and b) in Section~\ref{sec:mart}) we show using a martingale analysis that with high probability, each iterate satisfies the required neighborhood condition if the initialization is good enough.

We formalize the required condition for our per-iteration error analysis below:
\begin{definition}
	\label{def:It}
	We say that a sample path of the algorithm satisfies the condition $\sI_t$ at time $t$ if $\max_{i} \| \est[i]{t'} - \true[i] \| \leq \frac{C \sigma}{10}$ holds for all $0 \leq t' \leq t$. 
\end{definition}

\subsection{Error Reduction in Single Iteration} \label{sec:single}
Assume that for a given $t$, our estimates satisfy $\sI_t$, and let $\omega$ denote the current sample path of our algorithm. Let the current errors for each cluster be denoted by $\widetilde{E}^i_{t} =  \| \est[i]{t} - \true[i] \|^2$; Now, let $\widehat{E}^i_{t+1} = \E_{{\bf{x}^{t+N_0+1}}} \left[ \| \est[i]{t+1} - \true[i] \|^2 \, | \omega \right]$ by the \emph{expected error} of our new centers \emph{after a single streaming update for the next sample}. Let $\widetilde{V}_{t} = \max_{i} \widetilde{E}^i_{t}$ to be the maximum cluster error at time $t$. Finally, let $E^i_t = \E \left[ \| \est[i]{t} - \true[i] \|^2 \, | \, \sI_t \right]$ be the expected error conditioned on $\sI_t$, and let $E_t = \sum_i E^i_t$. Our main lemma toward showing Theorem~\ref{thm:kthm} is the following.

	\begin{lemma} \label{lem:per-cluster}
		If $\sI_t$ holds and $C \geq \Omega(\sqrt{\log k})$, then for all $i$, we have
		\begin{align*}
		\widehat{E}^{i}_{t+1} \leq& (1-\frac{\eta}{2k})  \widetilde{E}^i_t +  \frac{\eta}{k^5}  \widetilde{V}_t +
		O(1) \eta^2  d \sigma^2 + O(k) \eta (1-\eta) \exp(-C^2/8) (C^2+k) \sigma^2 \, .
		\end{align*}
	\end{lemma}

Now, using the above lemma (along with Theorem~\ref{thm:martingale}), we can get the following theorem.

\begin{theorem}
\label{thm:kmain}
Let $\gamma=O(k) {\eta^2 d}\sigma^2  + O(k^2) \eta (1-\eta) {\exp(-C^2/8) (C^2+k) \sigma^2 }$. Then if $C \geq \Omega(\sqrt{\log k})$, for all $t$, we have $E_{t+1} \leq (1 - \frac{\eta}{4k}) E_t + \gamma$. It follows that $E_N \leq (1-\frac{\eta}{4k})^N E_0+\frac{4k}{\eta}\gamma$.
\end{theorem}

\begin{proof}
	Let $\overline{E}^{i}_{t+1} = \E \left[\|\est[i]{t+1} - \true[i] \|^2  \, \Big{|} \sI_{t} \right]$ to be the average over all sample paths of $\widetilde{E}^i_{t+1}$ \emph{conditioned} on $\sI_t$. Recall that $E_{t+1}$ is very similar, except the conditioning is on $\sI_{t+1}$.  With this notation, let us take expectation over all sample paths where $\sI_t$ is satisfied, and use Lemma~\ref{lem:per-cluster} to get
	\begin{align*}
	\overline{E}^{i}_{t+1} \leq& (1-\frac{\eta}{2k})  {E}^i_t +  \frac{\eta}{k^5} E_t +
	O(1) \eta^2  d \sigma^2 + O(k) \eta (1-\eta) \exp(-C^2/8) (C^2 + k) \sigma^2 \, .
	\end{align*}
	And so, summing over all $i$ we will get
	\begin{align*}
	\overline{E}_{t+1} \leq& (1-\frac{\eta}{3k})  {E}_t +
	O(k) \eta^2  d \sigma^2 + O(k^2) \eta (1-\eta) \exp(-C^2/8) (C^2 + k) \sigma^2 \, .
	\end{align*}
	Finally note that $E_{t+1}$ and $\overline{E}_{t+1}$ are almost equal because of the following reasoning: $E_{t+1} \Pr \left[ \sI_{t+1} \right] \leq  \overline{E}_{t+1} \Pr \left[ \sI_{t} \right]$, and so $E_{t+1} \leq  \overline{E}_{t+1} (1+\frac{1}{N^2})$ since $\Pr \left[ \sI_{t+1} \right] \geq 1 - 1/N^5$ by our martingale Theorem~\ref{thm:martingale} proved in the following section.
\end{proof}

\begin{proof}[Proof of Theorem~\ref{thm:kthm}]
	From Theorem~\ref{thm:martingale} we know that the probability of $\sI_N$ being satisfied is $1 - 1/N^5$, and in this case, we can use Theorem~\ref{thm:kmain} to get the desired error bound. In case $\sI_N$ fails, then the maximum possible error is roughly $\max_{i,j} \|\ktrue{i} - \ktrue{j}\|^2 \cdot N$ (when all our samples are sent to the same cluster), which contributes a negligible amount to the bias term. 
\end{proof}

\begin{proof}[Proof sketch of Lemma~\ref{lem:per-cluster}]
	In all calculations in this proof, we first assume that the candidate centers satisfy $\sI_t$, and all expectations and probabilities are only over the new sample ${\bs{x}^{t+N_0+1}}$, which we denote by $\bs{x}$ after omitting the superscript.
Now recall our update rule: $\est[i]{t+1} = (1-\eta) \est[i]{t} + \eta {\bf{x}}$ if $\est[i]{t}$ is the closest center for the new sample ${\bf{x}}$; the other centers are unchanged. To simplify notations, let: \begin{equation}\label{eq:gt}
g_i^t({\bs{x}}) = 1 \text{ iff }i=\arg\min_j\|{\bf{x}}-\est[j]{t}\|,\ \ g_i^t({\bf{x}}) =0 \text{ otherwise}.\end{equation}
	By definition, we have for all $i$,
	\[
	\label{eqn:update_g}
	\est[i]{t+1} = (1-\eta) \est[i]{t} + \eta \left( g_i^t({\bs{x}}) {\bs{x}} + (1-g_i^t({\bs{x}})) \est[i]{t} \right) = \est[i]{t} + \eta g_i^t({\bs{x}}) (\bs{x} - \est[i]{t}).
	\]

Our proof relies on the following simple yet crucial lemmas. The first bounds the failure probability of a sample being closest to an incorrect cluster center among our candidates. The second shows that if the candidate centers are sufficiently close to the true centers, then the failure probability of mis-classifying a point to a wrong center is (upto constant factors) the probability of mis-classification even in the optimal solution (with true centers). Finally the third lemma shows that the probability of  $g^t_i({\bf{x}})=1$ for each $i$ is lower-bounded. Complete details and proofs appear in Appendix~\ref{app:update}.
\begin{lemma}
\label{lem:ibeatsj}
Suppose condition $\sI_t$ holds. For any $i$, $j \neq i$, let ${\bf{x}} \sim {\rm Cl}(j)$ denote a random point from cluster $j$. Then
$\Pr \left[ \|{\bf{x}} - \est[i]{t} \| \leq \|{\bf{x}} - \est[j]{t} \| \right] \leq \exp (-\Omega(C_{ij}))$.
\end{lemma}
\begin{lemma}
\label{cor:ibeatsj}
Suppose $\max (\| \est[i]{t} - \true[i] \|, \| \est[i]{t} - \true[i] \|) \leq  \sigma/C_{ij}$. For any $i$, $j \neq i$, let $x \sim {\rm Cl}(j)$ denote a random point from cluster $j$.  Then
$\Pr \left[ \|{\bf{x}} - \est[i]{t} \| \leq \|{\bs{x}} - \est[j]{t} \| \right] \leq O(1) \exp (-C_{ij}^2/8)$.
\end{lemma}
\begin{lemma}\label{lem:high-prob}
	If $\sI_t$ holds and $C = \Omega(\sqrt{\log k})$, then for all $i$, then $ \Pr \left[  g^t_i(\bs{x}) = 1 \right]  \geq \frac{1}{2k}$.
\end{lemma}

And so, equipped with the above notations and lemmas, we have
	\begin{align*}
	\widehat{E}^{i}_{t+1} &= \E_{\bs{x}} \left[\|\est[i]{t+1} - \true[i] \|^2  \right] \\
	&= (1-\eta)^2  \|\est[i]{t} - \true[i] \|^2  + \eta^2  \E \left[ \| g_i^t(\bs{x}) (\bs{x} - \true[i]) + (1-g_i^t(\bs{x})) (\est[i]{t} -  \true[i]) \|^2 \right] \\
	&+ 2 \eta (1-\eta) \E \left[ \Big{\langle} \est[i]{t} - \true[i] , \left( g_i^t(\bs{x}) (\bs{x} - \true[i])  \right. \right.+ \left. \left. (1-g_i^t(\bs{x})) (\est[i]{t} - \true[i]) \right) \Big{\rangle} \right] \\
	&\leq (1-\frac{\eta}{2k})  \widetilde{E}^i_t +
	\eta^2  \underbrace{\E \left[ \| g_i^t(\bs{x}) (\bs{x} - \true[i])  \|^2 \right]}_{T_1} + 2 \eta (1-\eta) \underbrace{\E \left[ \Big{\langle} \est[i]{t} - \true[i] , \left( g_i^t(\bs{x}) (\bs{x} - \true[i]) \right) \Big{\rangle} \right]}_{T_2}
	\end{align*}

	The last inequality holds because of the following line of reasoning: (i) firstly, the cross term in the second squared norm evaluates to $0$ due to the product $g_i^t(\bs{x}) (1 - g_i^t(\bs{x}))$, (ii) $\eta^2 \E \left[ (1- g_i^t(\bs{x})) \| \est[i]{t} - \true[i] \| ^2 \right] \leq  \eta^2 \widetilde{E}^i_t$, (iii) $2 \eta  (1-\eta) \E \left[ \langle \est[i]{t} - \true[i] , (1 - g_i^t(\bs{x})) (\est[i]{t} - \true[i])  \rangle \right]$ $\leq 2 \eta (1-\eta) \widetilde{E}^i_t \Pr \left[ g_i^t(\bs{x}) = 0 \right]$ $\leq 2 \eta (1-\eta) \widetilde{E}^i_t (1 - 1/2k)$ by Lemma~\ref{lem:high-prob}, and finally (iv) by collecting terms with coefficient $\widetilde{E}^i_t$.

The proof then roughly proceeds as follows: suppose in an ideal case, $g_i^t(\bs{x})$ is $1$ for all points $\bs{x}$ generated from cluster $i$, and $0$ otherwise. Then, if $\bs{x}$ is a random sample from cluster $i$, $T_1$ would be $d \sigma^2$, and $T_2$ would be $0$. Of course, the difficulty is that $g_i^t(\bs{x})$ is not always as well-behaved, and so the bulk of the analysis is in carefully using Lemmas~\ref{lem:ibeatsj}and~\ref{cor:ibeatsj}, and appropriately ``charging'' the various error terms we get to the current error $\widetilde{E}^i_t$, the variance, and the offline approximation error.
\end{proof}

\subsection{Ensuring Proximity Condition Via Super-Martingales}
\label{sec:mart}
In the previous section, we saw that condition $\sI_t = 1$
is sufficient to ensure that error reduces at time step $t+1$.
Our next key result is shows that $\sI_N = 1$ is satisfied with high probability.

\begin{theorem}
\label{thm:martingale}
Suppose $\max_i \| \kest{i}{0} - \ktrue{i} \| \leq \frac{C}{20}\sigma$,
then $\sI_N = 1$ w.p $\geq 1 -  (\frac{1}{{\rm poly }(N)} )$.
\end{theorem}

Our argument proceeds as follows.
Suppose we track the behaviour of the actual error terms $\widetilde{E}_t^i$ over time, and stop the process (call it a failure) when any of these error terms exceeds $C^2 \sigma^2/100$ (recall that they are all initially smaller than $C^2 \sigma^2/400$). Assuming that the process has not stopped, we show that each of these error terms has a super-martingale behaviour using Lemma~\ref{lem:per-cluster}, which says that on average, the expected one-step error drops. Moreover, we also show that the actual one-step difference, while not bounded, has a sub-gaussian tail. Our theorem now follows by using Azuma-Hoeffding type inequality for super-martingale sequences. 
While the high-level idea is reasonably clean, the details are fairly involved and we defer them to Appendix~\ref{app:martin}.

\section{Initialization for streaming k-means}\label{sec:init}
In Section~\ref{sec:analysis} we saw that our proposed streaming algorithm
can lead to a good solution for any separation $C \sigma = O(\sqrt{\log k}) \sigma$
if we can initialize all centers such that $\| \kest{i}{0} - \ktrue{i} \| \leq \frac{C}{20}\sigma$. {\sf InitAlg} (Algorithm~\ref{alg:init}) describes one such procedure, where we first approximately compute top-$k$ eigenvectors $U$ of the data covariance \emph{using a streaming PCA} algorithm \cite{HardtP14,mitliagkasC013} on $O(k^3d^3\log d)$ samples. We then store $k\log k$ points and project them onto the subspace spanned by $U$.  We then perform a simple distance based clustering \cite{vempalawang} that correctly classifies each of the point, under the separation assumption. We then re-estimate the means to obtain an initial estimate of $\ktrue{i}$, $1\leq i\leq k$. 
\begin{proof}[Proof of Theorem~\ref{thm:initialization}]
Using an argument similar to \cite{HardtP14} (Theorem 3), we get that $U$ obtained by the online PCA algorithm (Steps 1:4 of Algorithm~\ref{alg:init}) satisfies (w.p. $\geq 1-1/poly(d)$):
\begin{equation}
\|UU^T \ktrue{i}-\ktrue{i}\|^2\leq .01\sigma^2,\ \forall 1\leq i\leq k. \label{eq:init0}
\end{equation}
Now, let $\prtrue{i}=U^T\ktrue{i}$. For any ${\bf{x}}$ sampled from mixture distribution \eqref{eq:xt}, $U^T {\bf{x}}\sim \sum_i w_i \sN(\prtrue{i}, \sigma^2I)$. Hence, if $U^T {\bf{x}}^t$, $U^T{\bf{x}}^{t'}$ both belong to cluster $i$, then (w.p. $\geq 1-1/k^{\alpha}$):
\begin{align}
\|U^T{\bf{x}}^{t'}-U^T{\bf{x}}^{t'}\|^2=\|U^T (\bs{z}^{t}-\bs{z}^{t'})\|_2^2\leq (k+8\alpha\sqrt{k\log k})\sigma^2,\label{eq:init1}
\end{align}
where $\bs{x}^{t}=\ktrue{i}+\bs{z}^t$ and $\bs{x}^{t'}=\ktrue{i}+\bs{z}^{t'}$. The last inequality above follows by using standard $\chi^2$ random variable tail bound. Similarly if $U^T \bs{x}^t$, $U^T\bs{x}^{t'}$ belong to cluster $i$ and $j$, i.e.,  $\bs{x}^{t}=\ktrue{i}+\bs{z}^t$ and $\bs{x}^{t'}=\ktrue{j}+\bs{z}^{t'}$ then (w.p. $\geq 1-1/k^{\alpha}$):
\begin{align}
\|U^T\bs{x}^{t'}-U^T\bs{x}^{t'}\|^2=\|\prtrue{i}-\prtrue{j}\|^2+\|U^T (\bs{z}^{t}-\bs{z}^{t'})\|_2^2+2(\prtrue{i}-\prtrue{j})^TU^T(\bs{z}^{t}-\bs{z}^{t'})\notag\\
\geq (C^2-.2C+8\alpha \sqrt{k\log k}-16\alpha C\sqrt{\log k})\sigma^2,\label{eq:init2}
\end{align}
where the above equation follows by using \eqref{eq:init0}, setting $\alpha=C/32$ and using $C=\Omega((k\log k)^{1/4})$. 

Using \eqref{eq:init1}, \eqref{eq:init2}, w.h.p. all the points from the same cluster are closer to each other than points from other clusters. Hence, connected components of nearest neighbor graph recover clusters accurately.

Now, we estimate $\widehat{\boldsymbol{\mu}}_i=\frac{1}{|Cluster(i)|}\sum_{t\in Cluster(i)} U^T\bs{x}^t$ for each $i$. Since, our clustering is completely accurate, we have w.p. $ \geq 1-2m^2/k^{C/32}$,
\begin{equation} \|\widehat{\boldsymbol{\mu}}_i-\prtrue{i}\|_2\leq \sigma \frac{\sqrt{\log k}}{\sqrt{|Cluster(i)|}}.\label{eq:init3}\end{equation}

As $w_i=1/k$ for all $i$, $|Cluster(i)|\geq \frac{m}{k}-C\sqrt{\frac{m}{k}}$ w.p. $\geq 1-1/k^{C/32}$. Theorem now follows by setting $m=O(k \log k)$ and by using \eqref{eq:init0}, \eqref{eq:init3} along with $C=\Omega((k\log k)^{1/4})$.
\end{proof}

\begin{remark}
We would like to emphasize that our analysis for the convergence of streaming algorithms
works even for smaller separations $C = O( \sqrt{\log k} )$,
as long as we can get a good enough initialization. Hence, if a better initialization algorithm with weaker dependence of $C$ on $k$ would lead to an improvement in overall algorithm. 
\end{remark}

\section{Soft thresholding EM based algorithm}\label{sec:soft}
In this section, we study a streaming version of the Expectation Maximization (EM) algorithm ~\cite{dempster1977maximum} which is also used extensively in practice. While the standard $k$-means or Lloyd's heuristic is known to be agnostic to the distribution, and the same procedure can solve the mixture problem for a variety of distributions \cite{kumar2010clustering}, EM algorithms are designed specifically for the input mixture distribution.
 In this section, we consider a streaming version of the EM algorithm when applied to the problem of mixture of two spherical Gaussians with known variances. In this case, the EM algorithm reduces to a softer version of the Lloyd's algorithm where a point can be {\em partially} assigned to the two clusters. 
 Recent results by \cite{daskalakis2016ten, balakrishnan2014statistical, xu2016global} show convergence of the EM algorithm in the offline setting for this simple setup.
In keeping with earlier notation, let $\ktrue{1} = \true$ and $\ktrue{2} = -\true$ and the center separation $C = \frac{2 \| \true \|}{\sigma}$. Hence, $\bs{x^t}\stackrel{i.i.d}{\sim} \frac{1}{2} \sN(\true, \sigma^2I) + \frac{1}{2} \sN(-\true, \sigma^2I)$.

\begin{algorithm}
\caption{{\sf StreamSoftUpdate}$(N, N_0)$}
\label{alg:soft-streaming}
\begin{algorithmic}
\STATE {\bf Set} $\eta = \frac{3 \log N}{N}$.
\STATE {\bf Set} $\est[i]{0} \leftarrow{\sf InitAlgo}(N_0)$.
\FOR{$t = 1$ to $N$}
\STATE Receive $\bs{x}^{t + N_0}$ as generated by the input stream.
\STATE $\bs{x}=\bs{x}^{t+N_0}$
\STATE {\bf Let} $w_t = \frac {  \exp \big ( \frac{-\| \bs{x} - \est{t} \|^2}{\sigma^2} \big )} {\exp \big ( \frac{-\| \bs{x} - \est{t} \|^2}{\sigma^2} \big ) + \exp \big ( \frac{-\| \bs{x} + \est{t} \|^2}{\sigma^2} \big ) }$  \\
\vspace{10pt}
\STATE {\bf Set} $\est{t+1} = (1 - \eta) \est{t} + \eta [2w_t - 1] \bs{x}.$
\ENDFOR
\end{algorithmic}
\end{algorithm}
In our algorithm, $w_t(\bs{x})$ is an estimate of the probability that $\bs{x}$ belongs to the cluster with $\est{t}$,
given that it is drawn from a balanced mixture of gaussians at $\est{t}$ and $-\est{t}$.
Calculating $w_t(\bs{x})$ is like the E step and updating the estimate of the centers
is like the M step of the EM algorithm. Similar to the streaming Lloyd's algorithm presented in Section~\ref{sec:algo}, our analysis of streaming soft updates can be separated into
streaming update analysis and analysis {\sf InitAlg} (which is already presented in Section~\ref{sec:init}). We now provide our main theorem, and the proof is presented in Appendix~\ref{app:soft}.
\begin{theorem}[Streaming Update]
	\label{thm:softthm}
	Let $\bs{x}^t$, $1\leq t\leq N+N_0$ be generated using a mixture two balanced spherical Gaussians with variance $\sigma^2$. Also, let the center-separation $C \geq 4$, and also suppose our initial estimate $\est{0}$ is such that $\| \est{0} - \true \| \leq \frac{C \sigma}{20}$. 	Then, the streaming update of {\sf StreamSoftUpdate}$(N, N_0)$ , i.e, Steps 3-8 of Algorithm~\ref{alg:soft-streaming} satisfies: \[\E \left[ \| \est{N} - \true \|^2 \right]\leq \underbrace{\frac{\| \true \|^2}{N^{\Omega(1)}}}_{{\rm bias}} + \underbrace{O(1)\frac{ \log N}{N} d \sigma^2}_{\rm variance}.\]
\end{theorem}

\begin{remark}
	Our bias and variance terms are similar to the ones in Theorem~\ref{thm:main} but the above bound does not have the additional approximation error term. Hence, in this case we can estimate $\true$ consistently but the algorithm applies only to a mixture of Gaussians while our algorithm and result in Section~\ref{sec:algo} can potentially be applied to arbitrary sub-Gaussian distributions. \end{remark}
\begin{remark}
	We note that for our streaming soft update algorithm, it is not critical to know the
	variance $\sigma^2$ beforehand. One could get a good estimate of $\sigma$ by taking the mean of a random projection of a small number of points. We omit the details to simplify exposition of our proofs.
\end{remark}

\vspace*{-5pt}
\section{Conclusions}\vspace*{-5pt}
In this paper, we studied the problem of clustering with streaming data where each data point is sampled from a mixture of spherical Gaussians. 
For this problem, we study two algorithms that uses appropriate initialization: a) a streaming version of 
Lloyd's method, b) a streaming EM method. For both the methods we show that we 
can accurately initialize the cluster centers using an online PCA based method. We then show 
that assuming $\Omega((k\log k)^{1/4}\sigma)$ separation between the cluster centers, the updates by both 
the methods lead to decrease in both the bias as well as the variance error terms. 
For Lloyd's method there is an additional estimation error term, which even the offline algorithm incurs, and which  
is avoided by the EM method. 
However, the streaming Lloyd's method is agnostic to the data distribution and can 
in fact be applied to any mixture of sub-Gaussians problem. 
For future work, it would be interesting to study the streaming 
data clustering problem under deterministic assumptions like \cite{kumar2010clustering,monteleoniaistats16}. Also, it is an important question 
to understand the optimal separation assumptions needed for even the offline gaussian mixture clustering problem.
\clearpage
\bibliography{example_paper}
\bibliographystyle{plain}

\clearpage
\appendix
\section{Proofs from Section~\ref{sec:analysis}}
\label{app:update}

\begin{figure}[!tbp]
	\label{fig:kerror}
	\begin{center}
		\includegraphics[width=0.4\textwidth, height=0.2\textwidth]{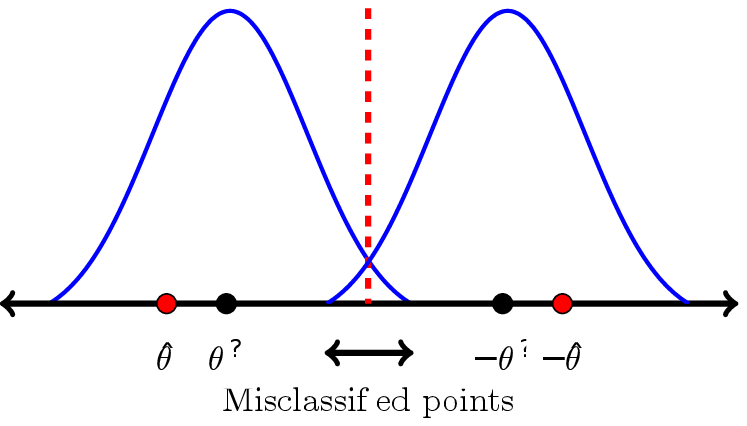}
		\caption{Illustration of optimal K-means error}
	\end{center}
\end{figure}

\begin{proof}[Proof of Lemma~\ref{lem:ibeatsj}]
Let $\bs{x} = \true[j] + \bs{z}$ where $\bs{z}$ is a mean-$\bs{0}$ spherical Gaussian with variance $\sigma^2 I$. Now, the condition $\|\bs{x} - \est[i]{t}\| < \|\bs{x} - \est[j]{t}\|$ is equivalent to:  $$	\|\true[j] + \bs{z} - \est[i]{t}\| < \|\true[j] + \bs{z} - \est[j]{t}\| ,i.e.,\ \|\true[j] - \true[i] + \bs{z} + \true[i] - \est[i]{t}\|^2 < \|\true[j] + \bs{z} - \est[j]{t}\|^2.$$ 
Now expanding the squared-norm, we get this to be equivalent to: $$\| \true[j] - \true[i]\|^2 + \|\true[i] - \est[i]{t}\|^2 + 2 \langle  \bs{z} , (\true[j] - \true[i]) + (\true[i] - \est[i]{t}) - (\true[j] - \est[j]{t}) \rangle + 2 \langle \true[j]-\true[i] , \true[i] - \est[i]{t} \rangle < \| \true[j] - \est[j]{t} \|^2.$$ 

Re-arranging terms, we get:
\begin{multline}
\|\bs{x} - \est[i]{t}\| < \|\bs{x} - \est[j]{t}\|\ \ \text{iff}\ \ 2\cdot \langle  \bs{z} , (\true[j] - \true[i]) + (\true[i] - \est[i]{t}) - (\true[j] - \est[j]{t}) \rangle \\< - \| \true[j] - \true[i]\|^2 - \|\true[i] - \est[i]{t}\|^2 - 2 \langle \true[j]-\true[i] , \true[i] - \est[i]{t} \rangle + \| \true[j] - \est[j]{t} \|^2.
\end{multline}
Using the Cauchy-Schwarz, we know that the above holds if the following is true: \begin{multline*}
2 \langle  \bs{z} , (\true[j] - \true[i]) + (\true[i] - \est[i]{t}) - (\true[j] - \est[j]{t}) \rangle < - \| \true[j] - \true[i]\|^2 - \|\true[i] - \est[i]{t}\|^2 + 2 \| \true[j]-\true[i] \| \| \true[i] - \est[i]{t} \| \\+ \| \true[j] - \est[j]{t} \|^2.\end{multline*}
Using $\sI_t$ and the fact that $\| \true[j] - \true[i] \|^2 = C_{ij}^2 \sigma^2$, the RHS above is at most $ -\frac{3}{4} C_{ij}^2 \sigma^2$. Moreover, the LHS is a Gaussian with mean $0$, and variance at most $ 4 \sigma^2 \| (\true[j] - \true[i]) + (\true[i] - \est[i]{t}) - (\true[j] - \est[j]{t}) \|^2 \leq  \sigma^2 ( C_{ij} \sigma + 2 (C\sigma/10) )^2$. So applying standard Gaussian concentration yields the desired result.
\end{proof}

\begin{proof}[Proof of Lemma~\ref{cor:ibeatsj}]
Following the proof of Lemma~\ref{lem:ibeatsj}, we get that using Cauchy-Schwarz inequality, the desired condition is stronger than the following condition: \begin{multline*}
2 \langle  \bs{z} , (\true[j] - \true[i]) + (\true[i] - \est[i]{t}) - (\true[j] - \est[j]{t}) \rangle < - \| \true[j] - \true[i]\|^2 - \|\true[i] - \est[i]{t}\|^2 \\+ 2 \| \true[j]-\true[i] \| \| \true[i] - \est[i]{t} \| + \| \true[j] - \est[j]{t} \|^2.
\end{multline*}
But now, since we have much better bounds on $\|\est[i]{t} - \true[i] \| < \sigma/C_{ij}$ and $\|\est[j]{t} - \true[j] \| < \sigma/C_{ij}$, we will get that the probability is almost equivalent to that of ${\cal N}(0, 4 C_{ij}^2 \sigma^4)$ random variable is smaller than $-C_{ij}^2 \sigma^2$, which in turn is $O(1) \exp(-C_{ij}^2/8)$.
\end{proof}

\begin{proof}[Proof of Lemma~\ref{lem:high-prob}]
	The proof follows directly from Lemma~\ref{lem:ibeatsj}. Consider the case when the next point $\bs{x}$ is sampled from cluster $i$. In this case, for every $j \neq i$, the probability that $g^t_{j}(\bs{x}) = 1$ is at most $1/2k$ by Lemma~\ref{lem:ibeatsj} since $C \geq \Omega(\sqrt{\log k})$. Then by the union bound, we get that that with probability at least $1/2$, $g^t_i(\bs{x}) = 1$. Now since the point $\bs{x}$ is sampled from cluster $i$ with probability $1/k$, the proof follows.
\end{proof}

\begin{proof}[Proof of Lemma~\ref{lem:per-cluster}]
	
	In all calculations in this proof, we first assume that the candidate centers satisfy $\sI_t$, and all expectations and probabilities are only over the new sample ${\bs{x}^t}$. For brevity in notation, we omit the $t$ superscript and simply refer to the random sample by $\bs{x}$.
	
	\begin{align*}
	\widehat{E}^{i}_{t+1} &= \E_{\bs{x}} \left[\|\est[i]{t+1} - \true[i] \|^2  \right] \\
	&= (1-\eta)^2  \|\est[i]{t} - \true[i] \|^2  + \eta^2  \E \left[ \| g_i^t(\bs{x}) (\bs{x} - \true[i]) + (1-g_i^t(\bs{x})) (\est[i]{t} -  \true[i]) \|^2 \right] \\
	&+ 2 \eta (1-\eta) \E \left[ \Big{\langle} \est[i]{t} - \true[i] , \left( g_i^t(\bs{x}) (\bs{x} - \true[i])  \right. \right.+ \left. \left. (1-g_i^t(\bs{x})) (\est[i]{t} - \true[i]) \right) \Big{\rangle} \right] \\
	&\leq (1-\frac{\eta}{2k})  \widetilde{E}^i_t +
	\eta^2  \underbrace{\E \left[ \| g_i^t(\bs{x}) (\bs{x} - \true[i])  \|^2 \right]}_{T_1} + 2 \eta (1-\eta) \underbrace{\E \left[ \Big{\langle} \est[i]{t} - \true[i] , \left( g_i^t(\bs{x}) (\bs{x} - \true[i]) \right) \Big{\rangle} \right]}_{T_2}
	\end{align*}
	The last inequality holds because of the following line of reasoning: (i) firstly, the cross term in the second squared norm evaluates to $0$ due to the product $g_i^t(\bs{x}) (1 - g_i^t(\bs{x}))$, (ii) $\eta^2 \E \left[ (1- g_i^t(\bs{x})) \| \est[i]{t} - \true[i] \| ^2 \right] \leq  \eta^2 \widetilde{E}^i_t$, (iii) $2 \eta  (1-\eta) \E \left[ \langle \est[i]{t} - \true[i] , (1 - g_i^t(\bs{x})) (\est[i]{t} - \true[i])  \rangle \right]$ $\leq 2 \eta (1-\eta) \widetilde{E}^i_t \Pr \left[ g_i^t(\bs{x}) = 0 \right]$ $\leq 2 \eta (1-\eta) \widetilde{E}^i_t (1 - 1/2k)$ by Lemma~\ref{lem:high-prob}, and finally (iv) by collecting terms with coefficient $\widetilde{E}^i_t$.
	The proof then follows from the below two lemmas.

	\begin{lemma}\label{lem:t1}
		If $\sI_t$ holds and $C = \Omega(\sqrt{\log k})$, then $T_1 \leq {O(d)} \sigma^2$.
	\end{lemma}

	\begin{lemma}
		\label{lem:t2}
		If $\sI_t$ holds and $C = \Omega(\sqrt{\log k})$, then $T_2 \leq  \frac{\widetilde{V}_t}{k^5} + O(k) \exp(- C^2/8) ( C^2 + k) \sigma^2$.
	\end{lemma}

\end{proof}
\begin{proof}[Proof of Lemma~\ref{lem:t1}]
		\begin{align*}
		T_1 &= \frac1k \sum_{j \neq i} \Pr \left[ g_i^t(\bs{x}) = 1 | \, x \sim {\rm Cl}(j) \right]{\E \left[ \| \bs{x} - \true[i] \|^2 \, \Big{|} g_t^t(\bs{x})=1 \, {\rm and }\, x \sim {\rm Cl}(j) \right]}, \\
		&+ \frac1k \Pr \left[ g_i^t(\bs{x}) = 1 | \, x \sim {\rm Cl}(i) \right] \E \left[ \| \bs{x} - \true[i] \|^2 \, \Big{|}  g_t^t(\bs{x})=1 \, {\rm and }\, x \sim {\rm Cl}(i) \right], \\
		&\leq \frac1k \sum_{j \neq i} O(1) \exp(-\Omega(C^2)) (C^2 + d) \sigma^2
		\\&\qquad\qquad+ \frac1k \Pr \left[ g_i^t(\bs{x}) = 1 | \, x \sim {\rm Cl}(i) \right]
		\E \left[ \| \bs{x} - \true[i] \|^2 \, \Big{|} \, g_t^t(\bs{x})=1 \, {\rm and }\, x \sim {\rm Cl}(i) \right], \\
		&\leq  O(1) \exp(-\Omega(C^2)) (C^2 + d) \sigma^2 + d \sigma^2.
		\end{align*}
		Above, the first inequality follows from Lemma~\ref{lem:exp-if-iwins-on-j} stated below, and the second inequality follows because we are summing a non-negative quantity over all $\bs{x} \sim {\rm Cl}(i)$ instead of only those where $g_i^t(\bs{x}) = 1$. 
\end{proof}

\begin{proof}[Proof of Lemma~\ref{lem:t2}]
		\begin{align*}
		T_2 =& \frac1k \sum_{j \neq i} \Pr \left[ g_i^t(\bs{x}) = 1 | \, x \sim {\rm Cl}(j) \right]
		{\E \left[ \Big{\langle} \est[i]{t} - \true[i] ,  (\bs{x} - \true[i])  \Big{\rangle} \, \Big{|} \,  g_t^t(\bs{x})=1 \, {\rm and }\, x \sim {\rm Cl}(j) \right]} \\
		&+ \frac1k \Pr \left[ g_i^t(\bs{x}) = 1 | \, x \sim {\rm Cl}(i) \right]
		\E \left[ \Big{\langle} \est[i]{t} - \true[i] ,  (\bs{x} - \true[i])  \Big{\rangle} \, \Big{|}  \, g_t^t(\bs{x})=1 \, {\rm and }\, x \sim {\rm Cl}(i) \right], \\
		=& \frac1k \sum_{j \neq i} \Pr \left[ g_i^t(\bs{x}) = 1 | \, x \sim {\rm Cl}(j) \right]
		{\E \left[ \Big{\langle} \est[i]{t} - \true[i] ,  (\bs{x} - \true[i])  \Big{\rangle} \, \Big{|} \,  g_t^t(\bs{x})=1 \, {\rm and }\, x \sim {\rm Cl}(j) \right]} \\
		&- \frac1k \Pr \left[ g_i^t(\bs{x}) = 0 | \, x \sim {\rm Cl}(i) \right]
		\E \left[ \Big{\langle} \est[i]{t} - \true[i] ,  (\bs{x} - \true[i])  \Big{\rangle} \, \Big{|} \, g_t^t(\bs{x})=0 \, {\rm and }\, x \sim {\rm Cl}(i) \right], \\
		\leq& \frac{\widetilde{V_t}}{k^5} + O(k)\exp(-C^2/8) (C^2 + k) \sigma^2.
		\end{align*}
		The inequality above follows from Lemmas~\ref{lem:ip-if-iwins-on-j}, and~\ref{lem:ip-if-iloses-on-i}. 
	\end{proof}

	\begin{lemma}
		\label{lem:exp-if-iwins-on-j}
		Suppose $\sI_t$ holds, and fix any $i, j\neq i$. Then, if $\bs{x} \sim {\rm Cl}(j)$, we have $\Pr_{\bs{x}} \left[ g_i^t(\bs{x}) = 1 \right]  {\E \left[ \| \bs{x} - \true[i] \|^2 \, \Big{|} \, g_t^t(\bs{x})=1 \right]} \leq O(1) \exp(-\Omega(C^2)) (C^2 + d) \sigma^2$.
	\end{lemma}

\begin{proof}
Intuitively, we know from Lemma~\ref{lem:ibeatsj} that if $\bs{x} \sim {\rm Cl}(j)$, then $\Pr_{\bs{x}} \left[ g_i^t(\bs{x}) = 1 \right] \leq \Pr_{\bs{x}} \left[ \|\bs{x} - \est[i]{t} \| < \|\bs{x} - \est[j]{t} \| \right]$, which in turn is at most $\exp(-\Omega(C_{ij}^2))$. In this case, we incur a cost of  $\| \bs{x} - \true[i] \|^2$ which is roughly $O(1) ( \| \true[j] - \true[i] \|^2 + \| \bs{x} - \true[j] \|^2 )$. The first term is $C_{ij}^2 \sigma^2$, and the second term in expectation is $d \sigma^2$. Multiplying this with the probability would establish the result. Of course, the expectation calculated here is not precise due to the conditioning involved. We formally now formally show the details.

As $\| \bs{x} - \true[i]\|^2$ is a non-negative quantity and the condition $\| \bs{x} - \est[i]{t} \| <  \| \bs{x} - \est[j]{t} \|$ is weaker than $g_i^t(\bs{x}) = 1$, we have: 
We have: \begin{multline*}
\Pr_{\bs{x}} \left[ g_i^t(\bs{x}) = 1  \right]  \cdot {\E \left[ \| \bs{x} - \true[i] \|^2 \, \Big{|}   \, g_t^t(\bs{x})=1 \right]} \leq \\\Pr_{\bs{x}} \left[ \| \bs{x} - \est[i]{t} \| <  \| \bs{x} - \est[j]{t} \|  \right] \cdot {\E \left[ \| \bs{x} - \true[i] \|^2 \, \Big{|}  \, \| \bs{x} - \est[i]{t} \| <  \| \bs{x} - \est[j]{t} \| \right]}.
\end{multline*}
So we bound the RHS above to complete the proof.

Now, let $\bs{x} = \true[j] + \bs{z}$ where $\bs{z}$ is sampled from a spherical normal Gaussian with mean $0$ and variance $\sigma^2$ along each direction. Then note that: $$\| \bs{x} - \true[i] \|^2 = \| \bs{z} + \true[j] - \true[i] \|^2 \leq 2 \|\bs{z}\|^2 + 2 C_{ij}^2 \sigma^2.$$ Moreover, we know by Lemma~\ref{lem:ibeatsj} that $\Pr_{\bs{x}} \left[ \| \bs{x} - \est[i]{t} \| <  \| \bs{x} - \est[j]{t} \|  \right] \leq \exp(-\Omega(C_{ij}^2))$. So multiplying the conditional expectation with the probability gives: \begin{multline}
\label{eq:up1}
\Pr_{\bs{x}} \left[ \| \bs{x} - \est[i]{t} \| <  \| \bs{x} - \est[j]{t} \|  \right] \cdot{\E \left[ 2 C_{ij}^2 \sigma^2 \, \Big{|}  \, \| \bs{x} - \est[i]{t} \| <  \| \bs{x} - \est[j]{t} \| \right]} \\\leq 2 C_{ij}^2 \sigma^2 \exp(-\Omega(C_{ij}^2)) \leq 2 C^2 \sigma^2 \exp(-\Omega(C^2)).
\end{multline}

We now bound: $2 \Pr_{\bs{x}} \left[ \| \bs{x} - \est[i]{t} \| <  \| \bs{x} - \est[j]{t} \|  \, \right] $ $ {\E \left[ \| \bs{z}  \|^2 \, \Big{|} \| \bs{x} - \est[i]{t} \| <  \| \bs{x} - \est[j]{t} \| \right]}$.  The crux of the proof now lies in the fact that the condition $\|\bs{x} - \est[i]{t} \|^2 < \|\bs{x} - \est[j]{t} \|^2$ boils down to the following linear inequality on $\bs{z}$: $$2 \langle \bs{z},  \est[j]{t} -  \est[i]{t} \rangle < \| \est[j]{t} \|^2 - \|\est[i]{t}\|^2  - 2 \langle \true[j] ,  \est[j]{t} -  \est[i]{t} \rangle.$$ Inspired by this, let us define $\tau := \frac{1}{2 \|\est[j]{t} -  \est[i]{t}\|} \left( \| \est[j]{t} \|^2 - \|\est[i]{t}\|^2  - 2 \langle  \true[j] ,  \est[j]{t} -  \est[i]{t} \rangle \right)$, and also define ${\bf a} := \frac{ \est[j]{t} -  \est[i]{t}}{\| \est[j]{t} -  \est[i]{t}\|}$. So now note that the condition $\|\bs{x} - \est[i]{t} \|^2 < \|\bs{x} - \est[j]{t} \|^2$ is equivalent to $\langle \bs{z}, {\bf a} \rangle < \tau$.

Since $\bs{z}$ is a spherical Gaussian, imposing a condition on $\langle \bs{z}, {\bf a} \rangle < \tau$ results in a truncated Gaussian along the direction ${\bf a}$ and an independent Gaussian in all orthogonal directions. So the expected squared length of the projection of $\bs{z}$ along orthogonal directions is  $\sigma^2$. Since we know that $\Pr_{\bs{x}} \left[ \| \bs{x} - \est[i]{t} \| <  \| \bs{x} - \est[j]{t} \|  \, \right]$ is at most $\exp(-\Omega(C_{ij}^2))$, their overall contribution is at most $\exp(-\Omega(C_{ij}^2)) d \sigma^2$. That is, 
$$\Pr_{\bs{x}} \left[ \| \bs{x} - \est[i]{t} \| <  \| \bs{x} - \est[j]{t} \|  \, \right]\cdot \E \left[ \| (I-{\bf a}{\bf a}^T)\bs{z}  \|^2 \, \Big{|} \| \bs{x} - \est[i]{t} \| <  \| \bs{x} - \est[j]{t} \| \right]\leq \exp(-\Omega(C_{ij}^2)) d \sigma^2.$$

So finally it remains to bound $\E \left[ \langle \bs{z}, {\bf a} \rangle^2 \, | \, \langle \bs{z}, {\bf a} \rangle < \tau \right]$. To this end, define the random variable $\hat{z} \equiv \langle \bs{z}, {\bf a} \rangle$. So we simply need to to upper bound $Pr \left[ \hat{z} < \tau  \, \right] \E \left[ \hat{z}^2 | \hat{z} < \tau \right]$. But now note that $Pr \left[ \hat{z} < \tau  \, \right] \leq \exp(-\Omega(C_{ij})^2)$. So using standard calculus, we get that this quantity attains a maximum value of $O(C_{ij}^2 \exp(-\Omega(C_{ij}^2))) \sigma^2$. This completes the proof.
\end{proof}

	\begin{lemma}
		\label{lem:ip-if-iwins-on-j}
		Suppose $\sI_t$ holds, and fix any $i, j\neq i$. Then, if $\bs{x} \sim {\rm Cl}(j)$, we have $\Pr_{\bs{x}} \left[ g_i^t(\bs{x}) = 1  \right]  {\E \left[ \langle \est[i]{t} - \true[i] , \bs{x} - \true[i] \rangle \, \Big{|} \, \, g_i^t(\bs{x})=1 \right]} \leq \frac{\widetilde{V}_t}{2 k^5} + O(1) \exp(-C^2/8) (C^2 + k)  \sigma^2$.
	\end{lemma}

\begin{proof}
Again we provide an intuitive proof sketch before giving the formal proof.
We first upper bound the inner product inside the expectation by $O(1) ( \|\bs{x} - \true[j]\|^2 + \|\true[i] - \true[j] \|^2 + \| \est[i]{t} - \true[i]\|^2)$. On average, the first term is $d \sigma^2$, the second term is $C_{ij}^2 \sigma^2$, and the third is at most $\widetilde{V}_t$, the maximum squared cluster error at time $t$. But note now that this quantity is multiplied by the probability $\Pr_{\bs{x}} \left[ g_i^t(\bs{x}) = 1 \right]$, which is at most $\Pr_{\bs{x}} \left[ \|\bs{x} - \est[i]{t} \| < \|\bs{x} - \est[j]{t} \| \right] \leq \exp(-\Omega(C_{ij}^2))$. 

So now we consider two cases: in the first case, $\widetilde{V}_t \geq \sigma^2/C^2_{ij}$. Here, we can charge all the three terms in terms of $\widetilde{V}_t$, and since $C \geq \Omega(\sqrt{\log k})$, the overall expression can be bounded by $\frac{\widetilde{V}_t}{2k^5}$; in the second case, $\widetilde{V}_t < \sigma^2/C^2_{ij}$. Here, we can use Lemma~\ref{cor:ibeatsj} instead of Lemma~\ref{lem:ibeatsj} to get a much more accurate failure probability of $\exp(-C_{ij}^2/8)$. So we now get a bound of $O(1) \exp(-C^2/8) (C^2 + d) \sigma^2$. Combining the two cases yields the lemma. We can improve upon the $d \sigma^2$ term with a more careful analysis.

{\bf Detailed Proof}:  Again, let $\bs{x} = \true[j] + \bs{z}$ where $\bs{z}$ is sampled from a spherical normal Gaussian with mean $0$ and variance $\sigma^2$ along each direction.
Firstly, note that: \begin{multline*}{\E \left[ \Big{\langle} \est[i]{t} - \true[i] ,  (\true[j] +\bs{z} - \true[i])  \Big{\rangle} \, \Big{|} \sI_{t} \, {\rm and } \, g_i^t({\bf \true[j] + \bs{z}})=1 \right]}$ $\leq \frac12 \left( \| \est[i]{t} - \true[i] \|^2 + \| \true[j] - \true[i] \|^2 \right) \\+  {\E \left[ \Big{\langle} \est[i]{t} - \true[i] ,  \bs{z}  \Big{\rangle} \, \Big{|} \sI_{t} \, {\rm and } \, g_i^t({\bf \true[j] + \bs{z}})=1 \right]}.\end{multline*}
Now, note that $g_i^t(\bs{x})=1$ depends on only the projection of $\bs{x}$ onto the subspace spanned by the $k+1$ vectors $\true[j]$ and the candidate centers $\est[j']{t}$ for all $j'$. So the projection of $\bs{z}$ on directions orthogonal to this subspace remain normal variables with mean $0$ and variance $\sigma^2$, and hence their contribution to the inner product is $0$. So we effectively only need to bound: \begin{multline*}{\E \left[ \Big{\langle} \est[i]{t} - \true[i] ,  {\bs{z}_{\Pi}}  \Big{\rangle} \, \Big{|} \sI_{t} \, {\rm and } \, g_i^t({\bf \true[j] + \bs{z}})=1 \right]}\\\leq\frac12 \| \est[i]{t} - \true[i] \|^2 + \frac12 \E \left[  \|{\bs{z}_{\Pi}}\|^2  \, \Big{|} \sI_{t} \, {\rm and } \, g_i^t({\bf \true[j] + \bs{z}})=1 \right].\end{multline*}
 So overall, the quantity we are seeking to bound in the Lemma is at most: \begin{multline*}\Pr_{\bs{x} \sim {\rm Cl}(j)} \left[ g_i^t(\bs{x}) = 1 \, \Big{|} \, \sI_{t} \right] \\\cdot \left(\| \est[i]{t} - \true[i]\|^2 + \frac12 \| \true[i] - \true[j]\|^2 + \E \left[  \|{\bs{z}_{\Pi}}\|^2  \, \Big{|} \sI_{t} \, {\rm and } \, g_i^t({\bf \true[j] + \bs{z}})=1 \right] \right).\end{multline*} Finally, it is easy to see that this expression is at most: \begin{multline*}\Pr_{\bs{x} \sim {\rm Cl}(j)} \left[ \|\bs{x} - \est[i]{t} \| < \|\bs{x} - \est[j]{t} \| \, \Big{|} \, \sI_{t} \right]\\ \cdot \left(\| \est[i]{t} - \true[i]\|^2 + \frac12 \| \true[i] - \true[j]\|^2 + \E \left[  \|{\bs{z}_{\Pi}}\|^2  \, \Big{|} \sI_{t} \, {\rm and } \, \|\bs{x} - \est[i]{t} \| < \|\bs{x} - \est[j]{t} \| \right] \right).\end{multline*}

We now consider two cases, depending on whether $\max (\| \est[i]{t} - \true[i] \|, \| \est[j]{t} - \true[j] \| ) >  \sigma/C_{ij}$ or not.

\medskip \noindent {\bf Case (i): $\max (\| \est[i]{t} - \true[i] \|, \| \est[j]{t} - \true[j] \| ) > \sigma/C_{ij}$.} In this case we will show that the desired expression is at most $\frac{\widetilde{V}_t}{k^5}$. Indeed, to this end, firstly note that we have $\| \true[i] - \true[j] \|^2 = C_{ij}^2 \sigma^2 \leq C_{ij}^4 \widetilde{V}_t$. Moreover, note that from Lemma~\ref{lem:ibeatsj} we have: $$\Pr_{\bs{x} \sim {\rm Cl}(j)} \left[ \|\bs{x} - \est[i]{t} \| < \|\bs{x} - \est[j]{t} \| \, \Big{|} \, \sI_{t} \right] \leq \exp(-\Omega(C_{ij})^2).$$ So using this, we can show in a manner akin to Lemma~\ref{lem:exp-if-iwins-on-j} that $\Pr_{\bs{x} \sim {\rm Cl}(j)} \left[ \|\bs{x} - \est[i]{t} \| < \|\bs{x} - \est[j]{t} \| \, \Big{|} \, \sI_{t} \right] \E \left[  \|{\bs{z}_{\Pi}}\|^2  \, \Big{|} \sI_{t} \, {\rm and } \, \|\bs{x} - \est[i]{t} \| < \|\bs{x} - \est[j]{t} \| \right]$ is at most $O(1) \exp(-\Omega(C_{ij})^2) (C_{ij}^2 + k ) \sigma^2 \leq \exp(-\Omega(C_{ij})^2) (C_{ij}^2 + k ) C_{ij}^2 \widetilde{V}_t$. Putting everything together, we get that the desired quantity we need to bound is at most $O(1) \exp(-\Omega(C_{ij})^2) \widetilde{V}_t \left(1 + C_{ij}^4   +  + k C_{ij}^2 \right)$. Overall this is at most $\frac{\widetilde{V}_t}{2k^5}$ since $C_{ij} \geq C = \Omega(\sqrt{\log k})$.

\medskip \noindent {\bf Case (ii): $\max (\| \est[i]{t} - \true[i] \|, \| \est[j]{t} - \true[j] \| ) < \sigma/C_{ij}$.}
In this case, we want to replace the use of Lemma~\ref{lem:ibeatsj} with Lemma~\ref{cor:ibeatsj} in the above proof. Indeed, from Lemma~\ref{cor:ibeatsj} we have: $$\Pr_{\bs{x} \sim {\rm Cl}(j)} \left[ \|\bs{x} - \est[i]{t} \| < \|\bs{x} - \est[j]{t} \| \, \Big{|} \, \sI_{t} \right] \leq \exp(-C_{ij}^2/8).$$ 
So using this, we can show in a manner akin to Lemma~\ref{lem:exp-if-iwins-on-j} that: \begin{multline*}\Pr_{\bs{x} \sim {\rm Cl}(j)} \left[ \|\bs{x} - \est[i]{t} \| < \|\bs{x} - \est[j]{t} \| \, \Big{|} \, \sI_{t} \right] \E \left[  \|{\bs{z}_{\Pi}}\|^2  \, \Big{|} \sI_{t} \, {\rm and } \, \|\bs{x} - \est[i]{t} \| < \|\bs{x} - \est[j]{t} \| \right]\\\leq O(1) \exp(-C_{ij}^2/8) (C_{ij}^2 + k ) \sigma^2.\end{multline*} 
Putting everything together, we get that the desired quantity we need to bound in the Lemma statement is at most $O(1) \exp(-C_{ij}/8^2) \sigma^2 \left( C_{ij}^2  + k \right)$.
\end{proof}

\begin{lemma}
		\label{lem:ip-if-iloses-on-i}
		Suppose $\sI_t$ holds. For any $i$ let $\bs{x} \sim {\rm Cl}(i)$. Then we have: $$\Pr_{\bs{x}} \left[ g_i^t(\bs{x}) = 0   \right]  {\E \left[ \langle \est[i]{t} - \true[i] , \bs{x} - \true[i] \rangle \, \Big{|} \, g_i^t(\bs{x})=0 \right]} \geq - \frac{\widetilde{V}_t}{2 k^5} -  O(k)  \exp(-C^2/8) (C^2 + k)  \sigma^2.$$
	\end{lemma}

\begin{proof}
Note that the expectation above is exactly: $${\E \left[ \Big{\langle} \est[i]{t} - \true[i] ,  (\true[i] +\bs{z} - \true[i])  \Big{\rangle} \, \Big{|} \, g_t^t({\bf \true[i] + \bs{z}})=0 \right]}= {\E \left[ \Big{\langle} \est[i]{t} - \true[i] ,  \bs{z}  \Big{\rangle} \, \Big{|} \, g_i^t({\bf \true[i] + \bs{z}})=0 \right]}.$$
Now, note that the condition $g_i^t(\bs{x})=0$ depends on only the projection of $\bs{x}$ onto the subspace spanned by the $k+1$ vectors $\true[i]$ and the candidate centers $\est[j']{t}$ for all $j'$. So the projection of $\bs{z}$ on directions orthogonal to this subspace remain normal variables with mean $0$ and variance $\sigma^2$, and hence their contribution to the inner product is $0$. So we effectively only need to bound ${\E \left[ \Big{\langle} \est[i]{t} - \true[i] ,  {\bs{z}_{\Pi}}  \Big{\rangle} \, \Big{|} \sI_{t} \, {\rm and } \, g_i^t({\bf \true[i] + \bs{z}})=0 \right]}$.
Here, ${\bs{z}_{\Pi}}$ denotes the projection of $\bs{z}$ onto the subspace spanned by the $k$ candidate centers and the true mean $\true[j]$. So overall, we get the following: 
$${\E \left[ \Big{\langle} \est[i]{t} - \true[i] ,  \bs{z} \Big{\rangle} \, \Big{|} \, g_t^t({\bf \true[i] + \bs{z}})=1 \right]}\geq -\frac12 \left( \| \est[i]{t} - \true[i] \|^2 + {\E \left[ \| {\bs{z}_{\Pi}}  \|^2 \, \Big{|} \, g_i^t({\bf \true[i] + \bs{z}})=0 \right]} \right).$$ Again, ignoring the effect of conditioning --- we deal with this in a manner identical to that in the proof of Lemma~\ref{lem:exp-if-iwins-on-j}---the LHS above is at least $ - \frac12 \left( \widetilde{V}_t + k \sigma^2 \right)$. Finally we note that $\Pr_{\bs{x} \sim {\rm Cl}(i)} \left[ g_i^t(\bs{x}) = 0 \right]$ is at most $\sum_{j \neq i} \Pr_{\bs{x} \sim {\rm Cl}(i)} \left[ \| \bs{x} - \est[i]{t} \| \geq \| \bs{x} - \est[j]{t} \| \right]$, which in turn is at most $k \exp(-\Omega(C^2))$ by Lemma~\ref{lem:ibeatsj}.
So multipyling these, and using two cases similar to the above proof, we get that the overall expression is at least $-\frac{\widetilde{V}_t}{2k^5} - O(k) \exp(-C^2/8) \left( C^2 + k \right) \sigma^2 $ as long as $C = \Omega(\sqrt{\log k})$.
\end{proof}

\section{Complete Details: Ensuring Proximity Condition Via Super-Martingales} \label{app:martin}
Our next key result is to show that a sample path through $N$ steps
satisfies $\sI_N = 1$ w.p $\geq 1 - \frac{1}{{\rm poly(N)}}$,
for suitable initialization. Recall that $\sI_N = 1$ if
that $\max_i \| \kest{i}{0} - \ktrue{i} \| \leq \frac{C \sigma}{10}$ for all $0 \leq t \leq N$.
In the rest of this section, we assume that the center separation $C \geq \Omega(\sqrt{\log k})$.

\begin{theorem}
\label{thm:mart}
Suppose our initial estimates $\kest{i}{0}$ satisfy
$\max_i \| \kest{i}{0} - \ktrue{i} \|\leq \frac{C \sigma }{20}$, then w.p $\geq 1 - \frac{1}{{\rm poly(N)}}, \sI_t = 1~\forall~ 1 \leq t \leq N$.
\end{theorem}
\begin{proof}
We first recall and define some useful quantities.
Firstly, $\tilde{E}_t^i = \| \kest{i}{t} - \ktrue{i} \|^2$ denotes the current squared error for cluster $i$ after the $t^{th}$ streaming update. It's also useful to define the quantity $e_t^i = \| \kest{i}{t} - \ktrue{i} \|$. Another important quantity is $\tilde{V}_t = \max_i \tilde{E}_t^i$,
and analogously $v^t = \max_i e_t^i$.
We will repeatedly use that $\tilde{V}^t = (v^t)^2$ and $\tilde{E}_t^i = (e_t^i)^2$.

Our argument proceeds as follows. Intuitively, if we start at $\frac{C \sigma}{20}$
within the true clusters, we need a lot of \textit{bad events} to keep moving the estimates
away and out of the $\frac{C \sigma}{20}$ radius ball. Since the samples ${\bf{x^t}}$ are independent,
we can use a concentration inequality to bound the probability with which a lot of bad
events occur together. At a high level, we can think of $Z(t) = e_{t+1}^i - e_t^i$ as
a martingale difference sequence on which we want to apply a concentration inequality. In order to do so, we perform the following steps.
\begin{itemize}
\item Bound $\E[ Z(t) \mid {\bf{x^1}}, {\bf{x^2}}, \hdots {\bf{x^t}}]$, and show conditions under which
this quantity is negative (since we want to show decrease). This is discussed in Lemma~\ref{lem:supermartingale}. At a high level, we show that the error term decreases on average in one step if the current error is sufficiently large (i.e., at least $C \sigma/20$). So we start a super-martingale series whenever the error term exceeds this value, and show that the probability of this series exceeding $C \sigma/10$ is negligible. We stop this series if the error falls below $C \sigma/20$, and start a new series when the error next exceeds this lower threshold of $C \sigma/20$. Since there can be at most $N k$ such series' (across clusters), a simple union bound would then suffice.
\item While the differences $Z(t)$ are not bounded, they have sub-gaussian tails. When a point is correctly
clustered, the error is roughly the norm of a Gaussian variable and is hence sub-gaussian. When a point is
incorrectly classified, the error is sub-gaussian. However the the mean is appproximately $C_{ij} \sigma$, which can be arbitrarily large. We deal with this by using Lemma~\ref{lem:ibeatsj} which says that the probability of misclassification is small. There are some more technical details
which are covered in Lemma~\ref{lem:subgaussian}.
\item As a next step, we use Azuma Hoeffding style inequality for sub-gaussians \citep{shamir2011variant}
to complete the proof (Lemma~\ref{lem:azuma}). However, we need to be careful while defining the martingale
sequences on which we apply this concentration inequality in order to satisfy the conditions required in Lemma~\ref{lem:supermartingale}
and Lemma~\ref{lem:subgaussian}. This forms the final part of the proof.
\end{itemize}

\begin{lemma}
\label{lem:supermartingale}
When $\frac{C \sigma}{20} \leq e_t^i$ and $v^t \leq \frac{C \sigma}{10}$, we have
$\E_{\bf{x^{t+1}}} [ e_{t+1}^i ] \leq e_t^i$.
\end{lemma}

\begin{proof}
We first show that $\E_{\bf{x^{t+1}}} [\tilde{E}_{t+1}^i] \leq \tilde{E}_t^i$.
The proof follows directly from Lemma~\ref{lem:per-cluster} (reproduced below)
\begin{align*}
\E_{\bf{x^{t+1}}} [ \tilde{E}_{t+1}^i] \leq& (1-\frac{\eta}{3k}) \widetilde{E}^i_t + \underbrace{\frac{\eta}{k^5} \widetilde{V}_t +
 O(1) \eta^2  d \sigma^2}_{ f(C, \eta, k)} + \underbrace{O(k) \eta (1-\eta) \exp(-C^2) C^2 \sigma^2}_{g(C, \eta, k)}  \, .
\end{align*}
Setting $\eta = \frac{ 3k \log 3N}{N}$.The term $f(C, \eta, k) + g(C, \eta, k) \leq \frac{\eta}{3k} \tilde{E}_t^i$
when $\tilde{V^t} \leq \frac{C^2 \sigma^2}{100}$ and $\tilde{E}_t^i \geq \frac{C^2 \sigma^2}{400}$.

Therefore, $\E_{\bf x} \left[ \tilde{E}_{t+1}^t \right] \leq \tilde{E}_t^i$.
By Jensen's inequality, we have $( \E_{\bf{x^{t+1}}}[e_{t+1}^i])^2 \leq \E_{\bf{x^{t+1}}} [\tilde{E}_{t+1}^i]$,
which in turn we showed is $\leq \tilde{E}_t^i$.
Taking square-roots, we get the required result.
\end{proof}


In order to bound the deviation from the mean, we appeal to Azuma style inequality for Subgaussians \cite{shamir2011variant}. We show next that $e^i_{t+1} - e_t^i$ has sub-gassian behaviour under
some conditions.
Since $\E_{\bf{x^{t+1}}}[ e_i^{t+1}]  - e_i^t$ is not zero, we'd
need the following lemma to deal with tail behaviour for non-zero mean variables.

\begin{lemma}
\label{lem:helper}
Suppose $X$ is a random variable that satisfies
$\Pr[X \geq a] \leq b_0 \exp \Big ( \frac{- a^2}{\sigma_0^2} \Big)$ for some $b_0 \geq 1$.
Then for any $\delta > 0$,
\begin{align*}
\Pr[X + \delta \geq a] \leq b_0 \exp \Big (\frac{\delta^2}{4 C^2 \sigma_0^2} \Big ) \exp \Big ( \frac{- a^2}{4 C^2 \sigma_0^2} \Big).
\end{align*}
Similarly, suppose we have $\Pr ( X \leq -a) \leq b_0 \exp \Big ( \frac{- a^2}{\sigma_0^2} \Big)$ for some $b_0 \geq 1$, then for any $\delta > 0$,
\begin{align*}
\Pr [X - \delta \leq -a] \leq b_0 \exp \Big (\frac{\delta^2}{4 C^2 \sigma_0^2}\Big)\exp \Big ( \frac{- a^2}{4 C^2 \sigma_0^2} \Big).
\end{align*}
\end{lemma}

\begin{proof}
For $ a \geq \delta$, we have $\Pr[X + \delta \geq a] = \Pr [ X \geq (a - \delta) ] = \exp \Big (\frac{ -(a - \delta)^2}{\sigma_0^2} \Big)$.

Hence we need to essentially show that
$(a - \delta)^2 \geq \frac{a^2}{4 C^2} - \frac{\delta^2}{4 C^2}.$

When $a \geq \frac{2 \delta c}{2c - 1}$, we have $(a - \delta)^2 \geq \frac{a^2}{4 C^2}$.

In the range, $\delta \leq a \leq \frac{2\delta c}{2c-1}$
the quadratic $(a - \delta)^2 + \frac{\delta^2}{4 C^2} - \frac{a^2}{4 C^2} \geq 0$.
This can be verified by noting that $a = \delta$ is the only root in the interval
and the condition is true at the end points of the interval.

An identical proof holds for bounding $\Pr [X - \delta \leq -a]$ when $a \geq \delta$.

When $ a \leq \delta$, the quantity $b_0 \exp \Big (\frac{ \delta^2 - a^2}{4 C^2 \sigma_0^2} \Big)$
is greater than $1$ and hence the result is trivially true in this case.
\end{proof}

We now study the tail behaviour of the quantity $e_{t+1}^i - e_t^i$.

\begin{lemma}
\label{lem:subgaussian}
The random variable $d_t^{i} = e_{t+1}^i - e_t^i$ has sub-gaussian tails,
when $v^t \leq \frac{C \sigma}{10}$. More precisely,
\begin{align*}
\forall a \geq 0, &\Pr [ d_t^{i} \geq a] \leq \exp \Big (- \frac{a^2}{\eta^2 \sigma^2 d} \Big), ~\text{and} \\
\forall a \geq 0, &\Pr [ d_t^i \leq - a] \leq \exp \Big (- \frac{a^2}{\eta^2 C^2 \sigma^2 d} \Big) \exp \Big(\frac{1}{100}\Big).
\end{align*}
\end{lemma}

\begin{proof}
For brevity in notation, we denote the random variable ${\bf x^{t+1}}$ as ${\bf x}$.
Also, let ${\bf w} = g_i^t({\bf{x}}){\bf{x}} + (1 - g_i^t({\bf{x}}))\kest{i}{t}$,
and let $\bf{z}$ denote a zero mean Gaussian in $d$ dimensions with variance $\sigma^2 I$.

We bound the probability that $d_t^i \geq a$, for some $a > 0$. We work through this in cases.
\paragraph{Case i.} Let ${\bf{x}} \sim Cluster(j), j \neq i$.
Recall from the update rule (Equation~\ref{eqn:update_g}),
\begin{align*}
e_{t+1}^i  - e_t^i &= \| (1 - \eta) \kest{i}{t} + \eta \bf{w} - \ktrue{i} \| - \| \kest{i}{t} - \ktrue{i} \| \\
                   &\leq  \eta \| \bf{w} - \kest{i}{t} \|.
\end{align*}
For convenience, let's consider the random variable $B = \| \bf{w} - \kest{i}{t} \|$
and bound $p = \Pr[ B \geq a' ]$.

Clearly, this term is $0$ when $g_i^t({\bf{x}}) = 0$, and $\| {\bf{x}} - \kest{i}{t} \|$
otherwise. Therefore, $\Pr[B \geq a' ] \leq \underbrace{\Pr[ g_i^t({\bf{x}}) = 1]}_{p_1}$
and $\Pr[ B \geq a' ] \leq \underbrace{\Pr [ \| {\bf{x}} - \kest{i}{t} \| \geq a']}_{p_2} $.

We can bound $p_2$ using Standard Gaussian concentration. Let ${\bf{x}} = \ktrue{j} + {\bf{z}}$.
Then $\| {\bf{x}} - \kest{i}{t} \| \leq \| {\bf{z}} \| + \frac{C \sigma}{10} + C_{ij} \sigma \leq \| {\bf{z}} \| + \frac{11}{10} C_{ij}\sigma$.

Therefore, $p_2 \leq \exp \big ( - \frac{(a' - \frac{11}{10}C_{ij}\sigma )^2}{\sigma^2 d} \big )$.

Hence, when $a' \geq \Omega(C_{ij}) \sigma$, we have $p \leq p_2 \leq \exp ( - \frac{a'^2}{2 \sigma^2 d})$.
From Lemma~\ref{lem:ibeatsj}, when $a' \leq \Omega(C_{ij}) \sigma$, $p \leq p_1 = \exp( -\Omega (C_{ij}^2) ) \leq \exp ( - \frac{a'^2}{\sigma^2})$.

The random variable of interest, $d_i^t = \eta B$.
Therefore, $\Pr[ d_i^t \geq a] \leq \exp \big ( - \frac{a^2}{ \eta^2 \sigma^2 d} \big)$.
\paragraph{Case ii.} Let ${\bf{x}} \sim Cluster(i)$.
Just as before, if $g_i^t({\bf{x}}) = 0$, $A = 0$.
However in the case that $g_i^t({\bf{x}}) = 1$, we bound $d_i^t$ slightly differently.
slightly differently.
\begin{align*}
d_i^t &=  \| (1 - \eta) \kest{i}{t} + \eta {\bf x} - \ktrue{i} \| - \|\kest{i}{t} - \ktrue{i} \| \\
&\leq - \eta \| \kest{i}{t} - \ktrue{i}  \|  + \eta \| \bf{x} - \ktrue{i} \| \\
&\leq \eta \| \bf{z} \|,
\end{align*}
where the last inequality follows from ${\bf{x}} = \ktrue{i} + {\bf{z}}$ in this case.
$\Pr [ A \geq a ] \leq \Pr [ \| {\bf{z}} \| \geq \frac{a}{\eta} ] \leq \exp \big ( - \frac{a^2}{ \eta^2 \sigma^2 d} \big)$ by standard Gaussian concentration.

We now look at bounding the negative tails. Just as before, we consider two cases.
\paragraph{Case i.}
Let ${\bf{x}} \sim Cluster(j), j \neq i$.
Recall from the update rule (Equation~\ref{eqn:update_g}),
\begin{align*}
e_{t+1}^i  - e_t^i &= \| (1 - \eta) \kest{i}{t} + \eta {\bf{w}} - \ktrue{i} \| - \| \kest{i}{t} - \ktrue{i} \| \\
                   &\geq  - \eta \| {\bf{w}} - \kest{i}{t} \|.
\end{align*}
Therefore, $\Pr [ A \leq  - a]  = \Pr [ B \geq \frac{a}{\eta} ]$.
This is the same quantity that we bounded for the positive tails.

\paragraph{Case ii.} Let ${\bf{x}} \sim Cluster(i)$.
In this case when $g_i^t({\bf{x}}) = 1$, we get
\begin{align*}
e_{t+1}^i - e_t^i &= \| (1 - \eta) \kest{i}{t} + \eta {\bf{z}}  \| - \| \kest{i}{t} - \ktrue{i} \|\\
                  &\geq - \eta \| \kest{i}{t} - \ktrue{i} \| - \eta \| {\bf{z}} \| \\
                  &\geq -\eta \frac{C \sigma}{10} - \eta \| {\bf{z}} \|,
\end{align*}
where the last inequality follows from the assumptions of the theorem.
$\Pr [ - \eta \| {\bf{z}} \| \leq -a ] \leq \exp \Big ( \frac{a^2}{\eta^2 \sigma^2 d} \Big)$.
Using Lemma~\ref{lem:helper} with $\delta = \frac{\eta C \sigma}{10}$,
for $a \geq 0$, we get $\Pr [ A \leq -a] \leq \exp \Big( \frac{1}{400 d} \Big) \exp \Big( - \frac{a^2}{4 \eta^2 C^2 \sigma^2 d} \Big).$
The expected error can decrease a lot when the point is correctly classified,
so we lose a factor in the sub-gaussian parameter while bounding the negative tails.
\end{proof}

From the above lemma, we obtain a useful corollary which is stated below.
\begin{lemma}
\label{lem:m-onestep}
If $e_t^i \leq \frac{C \sigma}{20}$, then with probability $1 - \frac{1}{\poly (N)}$
$e_{t+1}^i \leq \frac{3 C \sigma}{40}$.
\end{lemma}
\begin{proof}
Result follows from Lemma~\ref{lem:subgaussian} with setting $a = \frac{C \sigma}{40}$ and $\eta = \frac{ 3 k \log N}{N}$.
\end{proof}
Equipped with lemmas~\ref{lem:supermartingale} and~\ref{lem:subgaussian},
we can now put things together. We need to carefully handle the conditions under
which the above lemmas hold. This motivates the definition of the following random processes.

$A^i(t) = 1$ if $e_i^t \geq \frac{C \sigma}{20}$ and $e_i^{t-1} < \frac{C \sigma}{20}$.
In other words, $A^i(t)$ is set to $1$  when $e_i^t$ crosses the threshold of $\frac{C \sigma}{20}$
from below.

We also define $Z_\tau^i$ as follows.
\[
    Z^i_\tau(t) = \begin{cases}
      e_{t + 1}^i - e_t^i ~ \text{if} \sum\limits_{t'=1}^t A(t') = \tau \text{ and } v^t < \frac{C \sigma}{10}, \\
      0, \text{ Otherwise}
      \end{cases}
\]
In other words, $Z_\tau^i$ is ``active'' and tracks the difference in error with time, when $\tau$ is the first time that error crosses $\frac{C \sigma}{20}$ from below and has not crossed $\frac{C \sigma}{10}$.
It takes the value $0$ at all other times.

\begin{lemma}
\label{lem:mainsubgaussian}
For the sequence $Z_\tau^i(t) - \E_{\bf{x^{t+1}}}[Z_\tau^i(t)]$, there are constants $b > 1, c> 0$ such that
for all $\tau, i, t$ and any $a \geq 0$, it holds that
\begin{align*}
\Pr ( Z_\tau^i(t) - \E_{\bf{x^{t+1}}}[Z_\tau^i(t)] > a \mid \bf{x^1}, \hdots \bf{x^t}) &\leq b \exp (-c a^2),  \\
\Pr ( Z_\tau^i(t) - \E_{\bf{x^{t+1}}}[Z_\tau^i(t)] < -a \mid \bf{x^1}, \hdots \bf{x^t} )& \leq b \exp (-c a^2).
\end{align*}
\end{lemma}

\begin{proof}
The proof is mainly based on Lemma~\ref{lem:subgaussian}.
Recall from the definition of $Z_\tau^i$ that $Z_\tau^i(t) = 0$ if $e_t^i < \frac{C \sigma}{20}$
or $v^t > \frac{C \sigma}{10}$. The lemma is trivially true in this case.
Therefore, in the rest of the proof we assume that $\frac{C \sigma}{20} \leq e_i^t$
and $v^t \leq \frac{C \sigma}{10}$.

We know that $\E_{\bf{x^{t+1}}}[Z_\tau^i(t)] <0$ (From Lemma~\ref{lem:supermartingale}).

However, in order to study the tail behaviour of $Z_\tau^i(t) - \E_{\bf{x^{t+1}}}[Z_\tau^i(t)]$
we need to also lower bound the quantity $\E_{\bf{x^{t+1}}}[Z_\tau^i(t)]$.
Once again, we are only interested in the case that $\frac{C \sigma}{20} \leq e_i^t$
and $v^t \leq \frac{C \sigma}{10}$ and $Z_\tau^i(t) = e_{t+1}^i - e_t^i = d_t^i$.
From the update rule, we get
\begin{align*}
\E [ d_t^i] &= \| (\kest{i}{t} - \ktrue{i}) + \eta g(\bf{x}) [ \bf{x} - \kest{i}{t}] \| - \| \kest{i}{t} - \ktrue{i} \| \\
  &\geq -\eta \E [ g(\bf{x})\| \bf{x} - \ktrue{i} \|]  - \eta  \| \ktrue{i} - \kest{i}{t}\|  \\
  & \geq -\eta O(\sqrt{d}) \sigma - \eta \frac{ C \sigma}{10},
\end{align*}
where the last inequality follows from Lemma~\ref{lem:t1} and Jensen's inequality.

For ease of notation, let $\Delta = - \E[ d_t^i] \leq \eta O(\sqrt{d}) \sigma + \eta \frac{ C \sigma}{10}$.
Since we've now bounded the mean, we can use Lemma~\ref{lem:subgaussian} and
Lemma~\ref{lem:helper} to complete the proof as presented below.
\paragraph{Positive tail.}
Lemma~\ref{lem:helper} with $\delta = \Delta$, we get
$\Pr [ Z_\tau^i(t) - \E [ Z_\tau^i(t)] \geq a] \leq O(1) \exp \Big ( - \frac{a^2}{4 C^2 \eta^2 \sigma^2 d} \Big)$,
for $a \geq 0$.
\paragraph{Negative tail.}
$\Pr [ Z_\tau^i(t) - \E [ Z_\tau^i(t)] \leq -a] \leq \Pr [ Z_\tau^i(t) \leq - (a + \Delta)]  \leq \Pr [ Z_\tau^i(t) \leq - a] $.
From Lemma~\ref{lem:subgaussian}, $\forall a \geq 0$,
we have $\Pr [ Z_\tau^i(t) - \E [ Z_\tau^i(t)] \geq a] \leq O(1) \exp \Big ( - \frac{a^2}{4 C^2 \eta^2 \sigma^2 d} \Big)$.
\end{proof}

\begin{lemma}
\label{lem:azuma}
For all $\tau=1, 2, \hdots N$ and $i=1, 2, \hdots k$,
with probability $1 - \delta$,
\begin{align}
\label{eqn:azuma}
\sum\limits_{t=0}^{N-1} Z^i_\tau(t) \leq \sqrt {\frac{ \alpha C^2 \sigma^2 k^2 \log^2 N \log (\frac{1}{\delta})}{N} },
\end{align}
where $\alpha$ is some constant.
\end{lemma}
In particular, this gives us that $\Pr\Big[\sum\limits_{t=0}^{N-1} Z^i_\tau(t) \geq \frac{C \sigma}{40} \Big] \leq \frac{1}{\poly(N)}$.
\begin{proof}
We apply Azuma style concentration inequality to $Z^i_\tau(t) - \E_{\bf{x^{t+1}}}[Z^i_\tau(t)]$.

By definition, when $Z^i_\tau \neq 0$, the conditions for Lemma~\ref{lem:supermartingale}
and~\ref{lem:subgaussian} apply. Hence we can apply Theorem 2 of
\cite{shamir2011variant} with parameters $b = O(1)$ and $c = \frac{1}{C^2 d \sigma^4 \eta^2}$.
From Lemma~\ref{lem:supermartingale}, $\E[Z^i_\tau] \leq 0$ and hence the result.
\end{proof}

We say that a process $Z^i_\tau$ fails if there is a time $T$ such that $e_i^t > \frac{C \sigma}{10}$
and $Z^i_\tau(T) = e_t^{T+1} - e_i^T$, $\ie$ $Z_\tau^i$ was active at the instant $T$.
By definition of $Z^i_\tau$, $\sum \limits_{i=1}^N Z^i_\tau(t) = e^i_{T+1} - e^i_\tau$.

Recall that $\tau$ is the first time that $Z^i_\tau$ crosses $\frac{C \sigma}{20}$
from below. Lemma~\ref{lem:m-onestep} gives us that with probability $1 - \frac{1}{\poly(N)}$, $e^i_{\tau} \leq \frac{3 C \sigma }{40}$.
Taking union bound over both the bad events of Lemma~\ref{lem:m-onestep} and~\ref{lem:azuma},
we get that the probability that $Z^i_\tau$ fails is $\leq \frac{1}{\poly(N)}$.

For $\sI_t = 0$, atleast one of the process $Z^i_\tau$ has to fail (since we start with initialization
below $\frac{C \sigma}{20}$ and can't fail in a single step by Lemma~\ref{lem:m-onestep}). Taking union bound over the $k N$ such processes gives us the required result.
\end{proof}

\section{Complete Details: Soft streaming updates} \label{app:soft}
We first restate our main result here. 
\begin{theorem}[Streaming Update]
	Let $\bf{x^t}$, $1\leq t\leq N+N_0$ be generated using a mixture two balanced spherical Gaussians. Also, let the center-separation $C \geq 4$, and also suppose our initial estimate $\est{0}$ is such that $\| \est{0} - \true \| \leq \frac{C \sigma}{20}$. 
	
	Then, the streaming update of {\sf StreamSoftUpdate}$(N, N_0)$ , i.e, Steps 3-8 of Algorithm~\ref{alg:soft-streaming} satisfies: \[E_N \leq \underbrace{\frac{E_0}{\rm N^{\Omega(1)}}}_{{\rm bias}} + \underbrace{O(1)\frac{ \log N}{N} d \sigma^2}_{\rm variance} ,\]
	where $E_t = \E \left[ \| \est{t} - \true \|^2 \right]$. 
\end{theorem}

\begin{proof}
The proof for the soft streaming updates follows the same skeleton as that of 
streaming hard updates which is detailed in Section~\ref{sec:analysis}. 
We first bound the quantity $\widehat{E}_{t+1} = \E_{\bf{x^{t+1}}} \left[\| \est{t+1} - \true \|^2 \right]$ (analogous to Lemma~\ref{lem:per-cluster}). 

\begin{lemma}\label{lem:soft-onestep}
Suppose $\sI_t$ is satisfied, \ie $\| \est{t} - \true \| \leq \frac{C \sigma}{10}$ then,
\begin{align*}
\widehat{E}_{t+1} \leq \big (1 - \frac{3 \eta}{8} \big)\tilde{E_t} + O(1) \eta^2 d\sigma^2.
\end{align*}
\end{lemma}

\begin{proof}
For brevity of notation, we refer to $\bf{x^{t+1}}$ as $\bf{x}$. 
All expectations are taken with respect to $\bs{x}$. 
Recall the update rule for soft streaming. 
\begin{align*}
\est{t+1} &= (1 - \eta) \est{t} + \eta [2w_t({\bf x}) - 1] \bf {x}. 
\end{align*}
Expanding the expression for $\widehat{E}_{t+1}$, we get
\begin{align}
\widehat{E}_{t+1} &= \E \left[ \|\est{t+1} - \true \|^2 \right] \nonumber \\ 
                  &= (1 - \eta)^2 \tilde{E}_t  + 2 \eta(1 - \eta) \Big{\langle}\est{t} \true, \underbrace{\E \left[ (2 w^t({\bf{x}}) - 1) \bf{x} \right]}_{\bf{y^t}} - \true \Big{\rangle} ~+  \nonumber \\ \label{eqn:soft-analysis}
                  &\eta^2 \|\true\|^2 + \eta^2 \E \left[(2 w^t({\bf{x}}) - 1)\| \bf{x} \|^2\right] - 2 \eta^2 \Big{\langle} \true, \underbrace{\E\left[(2 w^t({\bf{x}}) - 1)\bf{x} \right]}_{\bf{y^t}} \Big{\rangle} 
\end{align}
From the above expression, we can see that the key quantity to bound is $\E \left[ (2 w^t ({\bf{x}}) - 1) \bf{x} \right] \eqdef \bf{y^t}$. 
The next lemma provides an expression for $\bf{y^t}$. 
\begin{lemma}
\label{lem:yt}
Suppose $\sI_t = 1$. We have the following expression for ${\bf{y^t}} = \E \left[ (2 w^t({\bf{x}}) - 1) \bf{x} \right]$. 
\begin{align*}
\bf{y^t} &= 2 \gamma^t (\est{t} - \true) + \true, ~~\text{where}  \\
\gamma^t &\leq \frac{1}{8C^2},
\end{align*}
\end{lemma}

\begin{proof}
For simplicity, let's define the following two terms. 
\begin{align*}
{\bf{y_1}^t} &\eqdef \E \left[ w^t({\bf{x}}) {\bf{x}} \mid {\bf{x}} \sim \sN(\true, \sigma^2I) \right]\\
{\bf{y_2}^t} &\eqdef \E \left[ w^t({\bf{x}}) {\bf{x}} \mid {\bf{x}} \sim \sN(-\true, \sigma^2I) \right]
\end{align*}
Note that ${\bf y^t} = \frac{1}{2} \Big ({\bf y_1^t} + {\bf y_2^t} \Big)$. 
Our first observation is that $\bf{y_2}^t$ takes the form $\gamma (\est{t} - \true)$ 
for some $\gamma \in \R$. 
\begin{align*}
\bf{y_2^t} &= \int\limits_{-\infty}^{\infty} p({\bf x}) w^t ({\bf x}) {\bf x} dx, \\
p({\bf x})w({\bf x}) &=  
\frac{1}{\sqrt{2 \pi \sigma^2}} \frac {  \exp \big ( \frac{-\| \bf{x} - \est{t} \|^2 - \| \bf{x} + \true \|^2}{\sigma^2} \big )} {\exp \big ( \frac{-\| \bf{x} - \est{t} \|^2}{\sigma^2} \big ) + \exp \big ( \frac{-\| \bf{x} + \est{t} \|^2}{\sigma^2} \big ) }.
\end{align*}
When $\langle \est{t} - \true, x \rangle = 0$, we have $p({\bf x}) w({\bf x}) = p(-{\bf x}) w(-{\bf x})$. 
Therefore, the terms cancel in the integration and we get $T_2 = \gamma (\est{t} - \true)$ for some $\gamma \in \R$. 

Our next observation is that $\bf{y_1^t} - \bf{y_2^t}  = \true$. 
This can be verified by seeing that for every term corresponding to $\true + \bf{z}$
in $\bf{y_1^t}$, there is a term corresponding to $- \true - \bf{z}$ in $\bf{y_2^t}$.
These two terms have the same multiplier $p({\bf x})$ but weight multipliers summing to one. 
Hence taking the difference of each term in ${\bf y_1^t}$ and ${\bf y_2^t}$ and integrating gives the required result. 
Since ${\bf y^t} = \frac{1}{2} \Big ({\bf y_1^t} + {\bf y_2^t} \Big)$, we get that
${\bf y^t} = 2 \gamma^t (\est{t} - \true) + \true$ where $\gamma^t$ is some constant. 

We now bound the value of this constant $\gamma^t$. 
Theorem 2 of \cite{daskalakis2016ten} gives the following bound on $\gamma^t$. 
(Our quantity ${\bf {y}}$ is the same as $\lambda^{t+1}$ and $\est{t}$ is $\lambda^t$ as per their notation). 
\begin{align}
\gamma^t &\leq \max \Bigg \{ \exp \Big ( \frac{-\| \est{t} \|^2}{2 \sigma^2} \Big), \exp \Big (\frac{- \langle \est{t}, \true \rangle^2}{2 \| \true \|^2 \sigma^2} \Big) \Bigg \}.
\end{align}
Suppose $\sI_t = 1$, then we have $\| \est{t} - \true \| \leq \frac{C \sigma}{10}$. 
That gives us that $\| \est{t} \| \geq \frac{9}{10} C \sigma$ and $\langle \est{t}, \true \rangle \geq \frac{99}{50}C^2 \sigma^2$. 
Together, this gives us that $\gamma^t \leq \frac{1}{8 C^2}$. 
\end{proof}
Now that we have an expression for $\bf{y^t}$, we can plug it back into \eqref{eqn:soft-analysis} 
to complete the proof. 

\begin{align*}
\widehat{E}_{t+1} &\leq (1 - \eta)^2 \tilde{E}_t + 4 \lambda^t \eta (1 - \eta)\tilde{E}_t \\
                  &+ \eta^2 \Big ( \sigma^2 d  + 2 \gamma^t \langle \true - \est{t}, \true \rangle \Big).
\end{align*}
The inequality follows from the fact that $(2 w^t({\bf x}) - 1) \leq 1 $ and 
$\E [ \| {\bf x} \|^2]  = \| \true \|^2 + d \sigma^2$. 
Since $2 \langle \true - \est{t}, \true \rangle \leq \| \est{t} - \true \|^2 + \| \true \|^2$, we get 
\begin{align*}
\widehat{E}_{t+1} &\leq \Big ( (1 - \eta)^2 + 4 \lambda^t \eta + \lambda^t \eta^2 \Big) \tilde{E}_t + \eta^2 \Big (\sigma^2 d + 2 \lambda^t \| \true \|^2 \Big) \\
                  &\leq \Big ( 1 - \frac{3 \eta}{8} \Big) \tilde{E}_t + O(1) \eta^2 \sigma^2 d, 
\end{align*}
where the last inequality follows from $\lambda^t \leq \frac{1}{8 C^2}$ (Lemma~\ref{lem:yt}).
\end{proof}
We have now shown error reduction in single iteration at step $t+1$, 
assuming that $\sI_t$ holds.

In order to complete the proof of the main theorem, we require the following martingale lemma. 
\begin{lemma}
Suppose our initial estimates $\est{0}$ satisfy 
$ \| \est{0} - \true \|\leq \frac{C \sigma }{20}$, then $\sI_t = 1$ w.p $1 - O (\frac{1}{{\rm poly }(N)} ), ~\forall~ 1 \leq t \leq N$.
\end{lemma}
\begin{proof}
This proof is identical to that of Theorem~\ref{thm:mart}, 
where we use Azuma Hoeffding inequality to bound the sum of independent sub-Gaussian random variables.  
The only change is that the random variable $g^t(\bf{x})$ is replaced by $2w^t({\bf{x}}) - 1$, 
while obtaining the sub-gaussian parameters. 
We omit the details from this presentation. 
\end{proof}

Using this martingale lemma, we can relate the quantity $\widehat{E}_{t+1}$ and $E_{t+1}$
similar to what we did for hard updates. Summing over $N$ steps, setting $\eta = \frac{3 \log N}{N}$ 
and observing that maximum error when $\sI_N =0 $ is $\| \true \|^2$ gives the final result. 
\end{proof}

\end{document}